\documentclass{article}

\usepackage[main, final]{neurips_2026}

\usepackage[utf8]{inputenc} 
\usepackage[T1]{fontenc}    
\usepackage{hyperref}       
\usepackage{url}            
\usepackage{booktabs}       
\usepackage{amsfonts}       
\usepackage{nicefrac}       
\usepackage{microtype}      
\usepackage{xcolor}         
\usepackage{amsthm}
\usepackage{amsmath}
\usepackage{enumitem}                 
\usepackage[most]{tcolorbox}          

\usepackage{tikz}
\usetikzlibrary{arrows.meta, decorations.pathmorphing, backgrounds, fit, calc, shapes.misc, positioning}
\usepackage{fontawesome5}

\definecolor{NodeBlue}{RGB}{60, 110, 200}
\definecolor{BathGray}{RGB}{160, 165, 175}
\definecolor{EdgeOrange}{RGB}{220, 130, 50}
\definecolor{InterventionGreen}{RGB}{75, 165, 100}
\definecolor{HamFrame}{RGB}{90, 90, 110}
\definecolor{HamBg}{RGB}{248, 248, 252}
\definecolor{LabelInk}{RGB}{60, 60, 75}

\definecolor{DecompBoxBg}{RGB}{235, 240, 250}
\definecolor{DecompBoxFrame}{RGB}{90, 110, 160}

\definecolor{IntuitionBg}{RGB}{255, 248, 220}
\definecolor{IntuitionFrame}{RGB}{180, 130, 50}
\definecolor{mygreen}{rgb}{0.8627450980392157, 0.9294117647058824, 0.7843137254901961}

\theoremstyle{plain}
\newtheorem{theorem}{Theorem}[section]
\newtheorem{proposition}[theorem]{Proposition}
\newtheorem{lemma}[theorem]{Lemma}

\theoremstyle{definition}
\newtheorem{definition}[theorem]{Definition}
\newtheorem{assumption}[theorem]{Assumption}

\theoremstyle{remark}
\newtheorem{remark}[theorem]{Remark}

\title{Reconciling Causality and Non-Equilibrium Thermodynamics with Hamiltonian Causal Models}

\author{%
  Dario~Rancati$^{1}$ \;
  Max~Welling$^{2,3}$ \;
  Francesco~Locatello$^{1}$ \\[0.5em]
  $^{1}$Institute of Science and Technology Austria \quad
  $^{2}$CuspAI \quad
  $^{3}$University of Amsterdam \\ [0.5em]
  Correspondence to: \texttt{dario.rancati@ist.ac.at}
}

\begin{document}

\maketitle
\begingroup
\renewcommand{\thefootnote}{}
\footnotetext{Code will be made available at
\url{https://github.com/dariorancati/hamiltonian-causal-models}.}
\endgroup

\begin{abstract}
\looseness=-1
Causal modeling of physical temporal phenomena must handle interventions that act along trajectories, nonstationary induced laws, path-dependent effects, and feedback mediated by dynamics, all challenging in standard causal models. We introduce \emph{Hamiltonian Causal Models} (HCMs), a trajectory-level framework in which observed variables interact with local environments and interventions act as controls of Hamiltonian mechanisms. HCMs separate immutable equations of motion from intervenable mechanisms and define causal effects as discrepancies between interventional path laws. A key motivation for HCMs is their natural interface with non-equilibrium thermodynamics. Entropy production quantifies the irreversibility of a process and is a central causal observable: it is estimable from data and witnesses causal effects along the system's evolution that are invisible to endpoint and cumulative versions of the standard average treatment effect. As in physics, cause and effect are not primitives of the relation between two random variables but arise from the non-invertibility of the thermodynamic arrow. With this, our paper reconciles the language of statistical causal models and non-stationary thermodynamics, offering new tools to describe causality in a wide range of physical systems. 
\end{abstract}

\section{Introduction}

Statistical approaches to causality aim to quantify how systems respond to
interventions rather than merely describing observed associations. While
correlations between variables are ubiquitous, causal relationships encode
how a joint distribution would change under external manipulations~\citep{spirtes2000causation,pearl2009causality,imbens2015causal}. This
interventional perspective has driven substantial developments across fields
including econometrics~\citep{imbens2015causal}, epidemiology~\citep{miguel2023causal}, machine
learning~\citep{scholkopf2021toward,peters2017elements}, and, most recently, AI for science~\citep{scholkopf2016modeling,feuerriegel2024causal,cadeiprediction,mencattini2026exploratory}. To accommodate the diverse needs of these
domains, multiple complementary frameworks have emerged, including potential
outcomes~\citep{imbens2015causal}, causal graphical models~\citep{pearl2009causality}, and structural causal models~\citep{spirtes2000causation}, but
also temporal approaches such as Granger causality~\citep{granger1969investigating} and differential
equation--based formulations~\citep{peters2022causal,lorch2024causal, BoekenMooij2024}. Despite this progress, no single existing framework natively captures the
combination of properties that arise naturally in physical dynamical
systems without ad-hoc constructions: interventions may be time-dependent controls acting along
trajectories, induced laws may be nonstationary, causal effects may be
complex and path-dependent, and feedback or cyclic dependence may unfold
through time.

\looseness=-1 These modeling challenges point naturally to the gap between statistical causal models and how causation emerges in physical systems.
As emphasized by
\citet{rovelli2023oriented}, in physical systems, causal direction
and temporal evolution are both tied to the notion of irreversibility. Indeed, many laws of physics are time-reversal symmetric: a solution to 
Newton's equations, for instance, remains a valid solution if all the velocities are
reversed. Causal experience is not symmetric in this way: if a stone is
thrown into a pond, we observe outgoing ripples; the time-reversed movie,
with ripples converging inward to the stone's entry point, is compatible
with the microscopic equations of motion but overwhelmingly unlikely at the
macroscopic level. This suggests that the causal influence of the rock on the pond is not simply encoded in the equations of motion but is tied to
the thermodynamic arrow of time. 

Non-equilibrium thermodynamics provides quantitative tools for measuring these notions, most
notably entropy production, which measures the irreversibility of a
trajectory by comparing it with its time reversal.

This connection becomes even sharper if we want to accommodate feedback and time-dependent interventions. In standard causal models, an intervention is idealized as
an arbitrary external operation that fixes, replaces, or perturbs a mechanism. At the same time, a causal model is only considered true when it aligns with randomized studies~\citep{peters2014causal}. In a
physical system, such an operation must be implemented by a controller or
protocol acting along the trajectory, and may exchange
information with the system doing so. Maxwell's demon \citep{sagawa2012informationthermodynamicsmaxwellsdemon} is the canonical example of a feedback intervention in
thermodynamics. The demon observes microscopic fluctuations and acts
conditionally on this information, apparently reducing the entropy of the
controlled system and thus contradicting the Second Law of Thermodynamics. The resolution is not that the intervention is
impossible, but that its thermodynamic analysis is incomplete unless the
controller's information (and hence entropic) cost is
included. The lesson for causal modeling is that time-dependent or adaptive
interventions are not mere assignments: they process information
from the trajectory and feed it back into the dynamics, thereby carrying
both causal and thermodynamic content.

In this paper, we introduce
\emph{Hamiltonian Causal Models} (HCMs), a framework for causal modeling of
controlled Hamiltonian and nonequilibrium processes. In HCMs, trajectories
are the primitive causal objects, and interventions are represented as
time-dependent control variables acting through the dynamics whose thermodynamic impact can be measured with non-equilibrium thermodynamic techniques. The model class we can describe is broad but structured. It includes familiar stochastic
dynamics such as controlled Langevin diffusions $
    dX_t = -\nabla U(X_t, \lambda_t)\mathrm{d}t + \sigma(\lambda_t)\mathrm{d}W_t $ while also allowing more general
controlled Hamiltonian and nonequilibrium evolutions. Overall, HCMs follow familiar causal assumptions such as independent mechanisms and markovianity. At the same time, they make it easy to model physical systems, without an ad-hoc characterization of time and dynamics and give us a language that naturally supports enticing properties: 

\begin{tcolorbox}[
  colback=IntuitionBg!12, colframe=IntuitionFrame,
  title=\textbf{Key Differentiating Properties of HCMs},
  fonttitle=\bfseries\small,
  sharp corners, boxrule=0.6pt,
  left=8pt, right=8pt, top=4pt, bottom=4pt,
]
\looseness=-1 HCMs provide a language in which one can natively study causal effects over {finite time
horizons, for stationary or nonstationary trajectories, in discrete or
continuous time}. Within this language, {interventions may be time-dependent,
ordered, adaptive, and mutually interacting; feedback loops and cyclic
dependence enter through temporal evolution;  assumptions such as
equilibrium initial conditions or a fully specified static graphical
structure are not imposed at the outset; and the distinction between causal
interventions and physical laws can be made explicit.}
\end{tcolorbox}

We summarize our contributions as follows:
\begin{itemize}[leftmargin=1.4em,itemsep=0.25em]
    \item \textbf{Hamiltonian Causal Models.}
    We introduce Hamiltonian Causal Models (HCMs), a trajectory-level
    causal framework for controlled Hamiltonian systems coupled to local
    environments. In HCMs, interventions are time-dependent policies acting
    on local Hamiltonian mechanisms, while causal effects are defined as
    discrepancies between interventional path laws.
    
    \item \textbf{Entropy production as a causal observable.}
    We develop the thermodynamic theory of HCMs. We show that work rates
    identify local causal interactions when the Hamiltonian is known, and
    that entropy production provides a path-space, data-estimable witness of
    causal effect. We further show how local 
    entropy production rates can be used to obtain a local causal influence criterion for
    nonstationary temporal dynamics.
    
    \item \textbf{Experimental validation.}
    We validate the theory on controlled dynamical systems. Our experiments
    show that entropy production detects path-level causal effects that can
    be invisible to endpoint or cumulative treatment-effect summaries, and
    that nodewise entropy-production rates can reveal the mechanism dependence structure induced by HCM interventions in cyclic dynamical systems.
\end{itemize}

\section{Hamiltonian Causal Models}
\label{sec:hcm}

This section introduces Hamiltonian Causal Models (HCMs) and their
notion of causal effect. We give the formal definitions and show that
HCMs describe a wide range of temporal phenomena, recovering several
existing trajectory-level causal models as special cases. Detailed
conditions and regularity hypotheses are deferred to
Appendix~\ref{app:hcm}.

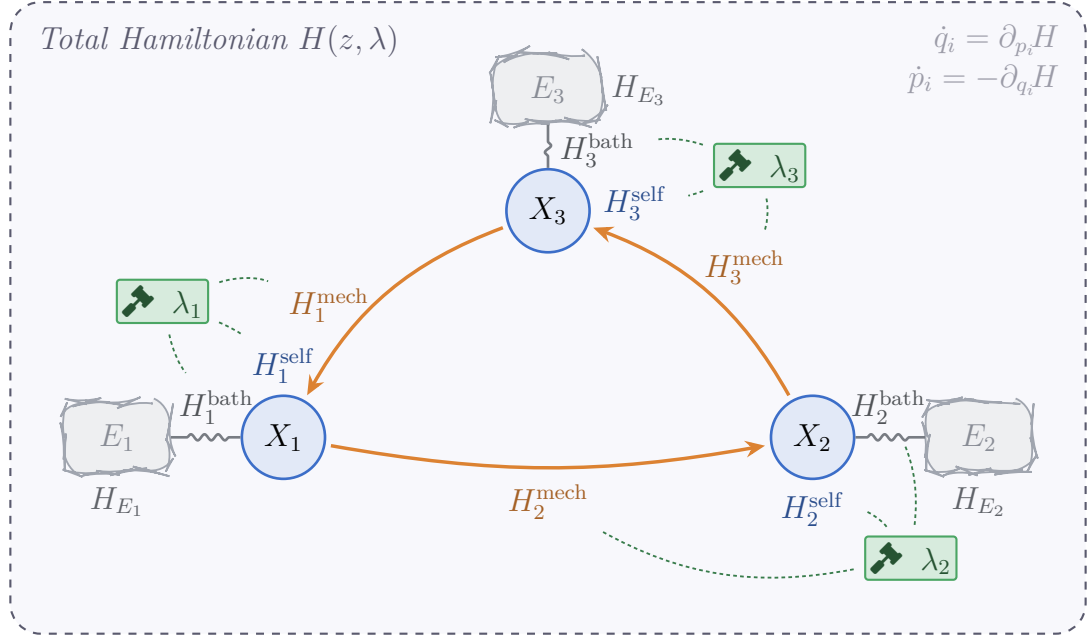
\begin{figure}[t]
\centering
\begin{tikzpicture}[
  >={Stealth[length=2.6mm, width=2.2mm]},
  every node/.style={font=\large},
  node/.style={
    circle, draw=NodeBlue, fill=NodeBlue!15, line width=1pt,
    minimum size=11mm, inner sep=0pt, font=\large\bfseries,
  },
  bath/.style={
    draw=BathGray, fill=BathGray!18, line width=0.7pt,
    decorate, decoration={random steps, segment length=2mm, amplitude=0.5mm},
    rounded corners=2.5mm,
    minimum width=14mm, minimum height=9mm, inner sep=1pt,
    font=\large\itshape, text=BathGray!85!black,
  },
  knob/.style={
    draw=InterventionGreen, fill=InterventionGreen!22, line width=0.8pt,
    rounded corners=1.2pt, minimum width=9mm, minimum height=5mm,
    font=\large\bfseries, text=InterventionGreen!50!black,
    inner sep=2pt,
  },
  knobline/.style={
    InterventionGreen!75!black,
    line width=0.7pt,
    dashed,
    dash pattern=on 1.4pt off 1.2pt,
    shorten >=7pt,
    shorten <=4pt,
  },
  bathchannel/.style={
    BathGray!75!black, line width=1pt,
    decorate, decoration={snake, amplitude=0.4mm, segment length=1.8mm,
    pre length=1.5mm, post length=1.5mm}
  },
  causaledge/.style={
    EdgeOrange, line width=1.3pt, ->,
    shorten >=2pt, shorten <=2pt
  },
  selflbl/.style={
    font=\large,
    text=NodeBlue!70!black,
    inner sep=1pt,
  },
  bathintlbl/.style={
    font=\large,
    text=BathGray!50!black,
    inner sep=1pt,
  },
  bathfreelbl/.style={
    font=\large\itshape,
    text=BathGray!60!black
  },
  mechlbl/.style={
    font=\large,
    text=EdgeOrange!75!black,
    inner sep=1pt,
  },
]

\node[node] (X1) at (-3.5, 0.6) {$X_1$};
\node[node] (X2) at ( 3.5, 0.6) {$X_2$};
\node[node] (X3) at ( 0.0, 3.6) {$X_3$};

\node[bath] (B1) at ($(X1)+(-2.2, 0)$) {$E_1$};
\node[bath] (B2) at ($(X2)+( 2.2, 0)$) {$E_2$};
\node[bath] (B3) at ($(X3)+( 0.0, 1.6)$) {$E_3$};

\node[bathfreelbl] at ($(B1)+(0, -0.85)$) {$H_{E_1}$};
\node[bathfreelbl] at ($(B2)+(0, -0.85)$) {$H_{E_2}$};
\node[bathfreelbl] at ($(B3)+(1.2, 0)$) {$H_{E_3}$};

\draw[bathchannel] (X1) -- (B1);
\draw[bathchannel] (X2) -- (B2);
\draw[bathchannel] (X3) -- (B3);

\node[bathintlbl] (BI1) at ($(X1)!0.5!(B1)+(0.21, 0.40)$) {$H_1^{\mathrm{bath}}$};
\node[bathintlbl] (BI2) at ($(X2)!0.5!(B2)+(-0.1, 0.40)$) {$H_2^{\mathrm{bath}}$};
\node[bathintlbl] (BI3) at ($(X3)!0.5!(B3)+(0.65, 0)$) {$H_3^{\mathrm{bath}}$};

\node[selflbl] (S1) at ($(X1)+(0, 0.95)$) {$H_1^{\mathrm{self}}$};
\node[selflbl] (S2) at ($(X2)+(0, -0.95)$) {$H_2^{\mathrm{self}}$};
\node[selflbl] (S3) at ($(X3)+(1.15, 0.1)$) {$H_3^{\mathrm{self}}$};

\draw[causaledge, <->] (X1) to[bend right=10] (X2);
\node[mechlbl] (M2) at ($(X1)!0.5!(X2)+(0,-0.85)$) {$H_2^{\mathrm{mech}}$};

\draw[causaledge] (X2) to[bend right=20] (X3);
\node[mechlbl] (M3) at ($(X2)!0.5!(X3)+(0.85,0.75)$) {$H_3^{\mathrm{mech}}$};

\draw[causaledge] (X3) to[bend right=20] (X1);
\node[mechlbl] (M1) at ($(X3)!0.5!(X1)+(-1.15,0.25)$) {$H_1^{\mathrm{mech}}$};

\node[knob] (L1) at ($(X1)+(-1.6, 1.8)$) {\faGavel\, $\lambda_1$};
\node[knob] (L2) at ($(X2)+(1.3, -1.6)$) {\faGavel\, $\lambda_2$};
\node[knob] (L3) at ($(X3)+(2.8, 0.6)$) {\faGavel\, $\lambda_3$};

\draw[knobline] (L1) to[bend left=8] (S1);
\draw[knobline] (L1) to[bend right=15] (BI1);
\draw[knobline] (L1) to[bend left=20] (M1);

\draw[knobline] (L2) to[bend right=25] (S2);
\draw[knobline] (L2) to[bend right=15] (BI2);
\draw[knobline] (L2) to[bend left=20] (M2);

\draw[knobline] (L3) to[bend left=8] (S3);
\draw[knobline] (L3) to[bend right=15] (BI3);
\draw[knobline] (L3) to[bend left=20] (M3);

\begin{pgfonlayer}{background}
  \node[
    draw=HamFrame, line width=0.8pt, dashed,
    fill=HamBg, rounded corners=4mm,
    fit={(B1) (B2) (B3) (L1) (L2) (L3)},
    inner sep=7mm,
  ] (Hframe) {};
\end{pgfonlayer}

\node[
  anchor=north west,
  font=\large\itshape,
  text=HamFrame,
  fill=HamBg,
  inner sep=3pt,
  rounded corners=1pt,
]
  at ($(Hframe.north west)+(0.25,-0.2)$)
  {Total Hamiltonian $H(z,\lambda)$};

\node[
  anchor=north east,
  font=\large\itshape,
  text=HamFrame!60,
  align=right,
  fill=HamBg,
  inner sep=3pt,
  rounded corners=1pt,
]
  at ($(Hframe.north east)+(-0.25,-0.2)$)
  {$\dot q_i = \partial_{p_i}\!H$\\[1pt]
   $\dot p_i = -\partial_{q_i}\!H$};

\end{tikzpicture}
\caption{Schematic Representation of an HCM: appropriate Hamiltonians are associated with every component and interaction among them, representing the energy stored in that block. Interventions are modelled to modify the Hamiltonian of each observable $X_i$ and of the incoming interactions.}
\end{figure}

\begin{tcolorbox}[
  colback=IntuitionBg!12, colframe=IntuitionFrame,
  title=\textbf{Intuition},
  fonttitle=\bfseries\small,
  sharp corners, boxrule=0.6pt,
  left=8pt, right=8pt, top=4pt, bottom=4pt,
]
A Hamiltonian Causal Model describes an \textbf{observed system} of $n$
variables, each paired with an \textbf{unobserved thermal bath} with which
it exchanges energy and information. A deterministic scalar function $H$,
the \emph{Hamiltonian}, drives the joint evolution of system and
environments through a prescribed set of microscopic equations of motion.
Stochasticity in the observed trajectories arises from two sources, both
extrinsic to the dynamics: randomness in the initial conditions, and the
fact that the environments are unobserved. The environments thus play the
role of \textbf{exogenous noise} variables of classical Pearl causality. The Hamiltonian decomposes into terms associated with each component of the system: an intrinsic potential per variable, couplings to other observed nodes, couplings to the baths, and the baths' own internal dynamics. In this sense, the Hamiltonian's decomposition plays the role of the \textbf{causal Markov factorization}
from classical structural causal models. \textbf{Interventions} modify
$H$, and through it the mechanisms encoded in the decomposition, without
altering the equations of motion themselves.
\end{tcolorbox}

\vspace{5pt}

\begin{definition}[Hamiltonian Causal Model]
\label{def:hcm}
Fix an horizon $T\in\mathbb{R}_{+}\cup\{+\infty\}$. An
\emph{Hamiltonian Causal Model} is a tuple
$\mathcal{M}=(\mathcal{P},z_0,\Gamma,\omega,\mathcal{L},\mathcal{A},H)$
consisting of:
\begin{enumerate}[label=\textbf{(\arabic*)},leftmargin=1.6em,itemsep=0.35em]

\item A probability space $\mathcal{P}=(\Omega,\mathcal{F},\mathbb{P})$
and a random initial state $z_0:\Omega\to\Gamma$.

\item A symplectic manifold $(\Gamma,\omega)$ with a node-wise
factorization
$\Gamma=\Gamma_X\times\Gamma_E,\;
 \Gamma_X=\textstyle\prod_{i=1}^n\Gamma_{X_i},\;
 \Gamma_E=\textstyle\prod_{i=1}^n\Gamma_{E_i}$,
where $\Gamma_{X_i}$ is the state space of the observed variable $X_i$
and $\Gamma_{E_i}$ the state space of its bath.

\item An intervention manifold $\mathcal{L}=\prod_{i=1}^n\mathcal{L}_i$
and a class $\mathcal{A}$ of admissible policies
$\lambda:[0,T]\times\Omega\to\mathcal{L}$.
\item A smooth Hamiltonian $H:\Gamma\times\mathcal{L}\to\mathbb{R}$
admitting the node-wise decomposition:
\end{enumerate}
\vspace{-0.2em}
\begin{tcolorbox}[colback=DecompBoxBg!18,colframe=DecompBoxFrame,
sharp corners,boxrule=0.5pt,left=8pt,right=8pt,top=4pt,bottom=4pt]
\vspace{-1em}
\begin{equation}
\label{eq:hcm-decomposition}
  H(z,\lambda)
  =
  \underbrace{H_E(e)}_{\text{free environments}}
  +
  \sum_{i=1}^{n}
  \Bigl(
    \underbrace{H_i^{\mathrm{self}}(x_i,\lambda_i)}_{\text{intrinsic}}
    +
    \underbrace{H_i^{\mathrm{mech}}(x_i,x_{-i},\lambda_i)}_{\text{coupling to other nodes}}
    +
    \underbrace{H_i^{\mathrm{bath}}(x_i,e_i,\lambda_i)}_{\text{coupling to environment}}
  \Bigr)
\end{equation}
\vspace{-0.9em}
\end{tcolorbox}
\vspace{-0.2em}
\begin{enumerate}[label=\textbf{(\arabic*)},leftmargin=1.6em,itemsep=0.35em]
\setcounter{enumi}{4} 
\item Dynamics generated by $H$: setting $H_\lambda(t,z):=H(z,\lambda_t)$,
the system evolves under the controlled flow
$\Phi_t^\lambda:\Gamma\to\Gamma$ defined by the unique vector field
$Y_{H_\lambda}$ satisfying
$\omega(Y_{H_\lambda}(t,\cdot),\cdot)=\mathrm{d}H_\lambda(t,\cdot)$. The
\emph{forward observed path law} is
$\mathbb{P}_F^\lambda:=(\Pi_X\circ\Psi^\lambda)^{\#}\mathbb{P}$, where
$\Psi^\lambda(\omega)(t):=\Phi_t^\lambda(z_0(\omega))$ and $\Pi_X$
projects onto observed coordinates.
\end{enumerate}

We assume regularity sufficient for the controlled flow to be globally
defined on $[0,T]$, factorization of the initial environment law across
nodes, and existence of a time-reversal compatible with $H$; precise
hypotheses are stated in Appendix~\ref{app:HCM-hypotheses}.
\end{definition}

\looseness=-1\textbf{Causal semantics. \ }
As anticipated, the baths $E_1,\dots,E_n$ play the role of exogenous
noise. On interventions, we adopt the point of view
of Phenomenological Causality \citep{janzing2022phenomenologicalcausality}, consistent with other works in causality indexing interventions as part of the causal model~\citep{mooij2020joint}: causal quantities
are defined relative to a set of actions available to an operator,
formalized here as the admissible class $\mathcal{A}$ of adapted
intervention policies. Through $\lambda_i$, the operator can modify any
of the three node-$i$ contributions in \eqref{eq:hcm-decomposition}.
For example, by reshaping the strength of $X_i$'s couplings to
the other observed nodes through $H_i^{\mathrm{mech}}$, or by decoupling $X_i$
from its bath through $H_i^{\mathrm{bath}}$, effectively detaching the
subsystem from its exogenous noise. Crucially, interventions modify the
Hamiltonian function but never the form of Hamilton's equations, separating
intervenable mechanisms from immutable dynamical law. 

The Hamiltonian decomposition therefore plays a role analogous to the
 independent-causal-mechanisms assumption in structural causal modeling. However, there is a striking difference with the usual definition of causal parenthood in DAGs and SCMs. In Hamiltonians, the causal semantics is endowed by the interventional policies and not just by the additive decomposition of the Hamiltonian, unlike, e.g., the structural assignment of an SCM. That two or more variables exchange information via $H_i^{\mathrm{mech}}$ can be, on its own, insufficient to determine the direction of causality.


\begin{definition}[Mechanism dependence]
\label{def:hcm-adjacent}
The mechanism dependence set of node $i$ is
$\mathcal D_i
:=\{j\neq i:
\partial_{x_j}H_i^{\mathrm{mech}}\not\equiv0\}$.
Nodes $X_i$ and $X_j$ are \emph{mechanically paired}
if $j\in\mathcal D_i$ or $i\in\mathcal D_j$.
\end{definition}

The following result, proven in Appendix~\ref{app:proofs}, shows
that HCMs describe a wide class of stochastic dynamics and
recover classical structural causal models as a special case.

\begin{theorem}[Realization, informal]
\label{thm:realization-informal}
Let $U:\mathbb{R}^{n}\times\mathcal{L}\to\mathbb{R}$ be a smooth
potential and $\sigma:\mathcal{L}\to\mathbb{R}^{n\times n}$ a smooth,
diagonal, positive diffusion coefficient. Then there exists a sequence
of HCMs whose observed coordinates have, in the limit of large bath
size and small inertial parameter, path law on
$C([0,T];\mathbb{R}^{n})$ equal to that of the It\^o diffusion
\begin{equation}
\label{eq:vector-sde}
  \mathrm{d}X_t \;=\; -\nabla_x U(X_t,\lambda_t)\,\mathrm{d}t
  \;+\; \sigma(\lambda_t)\,\mathrm{d}W_t,
\end{equation}
where $W$ is a standard $n$-dimensional Brownian motion independent of
the initial state $X_0$.
\end{theorem}

\paragraph{Static causal models as a stationary special case.}
Choosing $\sigma(\lambda_t) = sI$, $U=-\tfrac{1}{2}s^{2}\log p$ for a smooth strictly
positive density $p$ and initializing
$X_0\sim p$ and $\lambda \equiv 0$, the HCM realizes a stationary
diffusion with observational
distribution $p$. We thus recover classical structural causal models,
together with their cyclic generalizations as defined in \citet{lorch2024causal},
as the stationary regime of an HCM. Of note, by tuning the set of admissible policies $\mathcal{A}$ one can also define a notion of interventions that takes the dynamic response to $\lambda$ into account.

With HCMs in place, we turn to the notion of causal effect they
support. Let
$\mathbb{P}_i^\lambda := (\Pi_i)_\#\mathbb{P}_F^\lambda$ denote the
\emph{node-wise marginal path law} of $X_i$ under policy $\lambda$,
obtained by projecting $\mathbb{P}_F^\lambda$ onto the $i$-th
coordinate.

\begin{tcolorbox}[
  colback=IntuitionBg!12,
  colframe=NodeBlue,
  fonttitle=\bfseries\small,
  sharp corners,
  boxrule=0.7pt,
  left=8pt,
  right=8pt,
  top=5pt,
  bottom=5pt,
]
\begin{definition}[Causal effect]
\label{def:hcm-causal-effect}
Fix $i\neq j$ and a baseline policy $\lambda_{-j}$ for all intervention
coordinates except $\lambda_j$. We say that $X_j$ has a
\emph{causal effect} on $X_i$ over $[0,T]$ if there exist two
admissible policies $\lambda_j^{(1)},\lambda_j^{(2)}$ such that $
\mathbb{P}_i^{(\lambda_j^{(1)};\lambda_{-j})}
\;\neq\;
\mathbb{P}_i^{(\lambda_j^{(2)};\lambda_{-j})}.$
\end{definition}
\end{tcolorbox}

Combining Definition~\ref{def:hcm-causal-effect} and the notion of mechanical pairing of Definition~\ref{def:hcm-adjacent}, we can define what parenthood is for an HCM.

\begin{definition}[Operational parents]
\label{def:hcm-operational-parents}
For $i\neq j$, $X_j$ is a \emph{parent} of $X_i$,
written $j\to i$, if the nodes are mechanically paired
according to Definition~\ref{def:hcm-adjacent}
and $X_j$ has a causal effect on $X_i$ in the sense of
Definition~\ref{def:hcm-causal-effect}.
\end{definition}

This graph records intervention-induced influence between
mechanistically paired nodes. Such a pair can have both
arrows, one arrow, or neither. Definition~\ref{def:hcm-causal-effect} is intrinsically
trajectory-level: causal effect is a change in the entire path law of
$X_i$, not in a single-time observable. It is natural to quantify it
through discrepancies between the induced path laws.

\begin{definition}[Generalized trajectory ATE]
\label{def:hcm-ate}
Let
$\mathfrak{D}:\mathcal{P}\!\left(C([0,T],\Gamma_{X_i})\right)
\times \mathcal{P}\!\left(C([0,T],\Gamma_{X_i})\right) \to \mathbb{R}$
be a measurable discrepancy on pairs of path laws. For two admissible
policies $\lambda^{(1)},\lambda^{(2)}\in\mathcal{A}$, the
\emph{generalized trajectory average treatment effect} on node $X_i$ is
\[
\mathrm{ATE}_{0:T}^{(i)}(\mathfrak{D};\lambda^{(1)},\lambda^{(2)})
\;:=\;
\mathfrak{D}\!\left(\mathbb{P}_i^{\lambda^{(1)}},
\mathbb{P}_i^{\lambda^{(2)}}\right).
\]
\end{definition}

When $\mathfrak{D}$ separates measures
($\mathfrak{D}(\mu,\nu)=0\iff\mu=\nu$), $X_j$ has a causal effect on
$X_i$ in the sense of Definition~\ref{def:hcm-causal-effect} if and
only if $\mathrm{ATE}_{0:T}^{(i)}\neq 0$ for some pair of admissible policies; we
record this formally in Proposition~\ref{prop:effect-iff-discrepancy}
of Appendix~\ref{app:proofs}.

\paragraph{Relation to existing notions of causal effect.}
Definition~\ref{def:hcm-ate} unifies several existing notions of
temporal causal effect through suitable choices of $\mathfrak{D}$.
\emph{Single-time effects}
$\mathbb{E}[f_t(X_i^\lambda(t))]-\mathbb{E}[f_t(X_i^{\lambda'}(t))]$
appear in two well-developed strands of the literature: fixed-horizon
and regime-comparison effects, as in longitudinal g-methods, marginal
structural models, dynamic treatment regimes, and target-trial
formulations \citep{RobinsHernanBrumback2000, HernanBrumbackRobins2000,
NaimiColeKennedy2017, Murphy2003, HernanRobins2016,
HernanWangLeaf2022}; and local or lagged effects in structural nested
mean models, causal excursion effects, single-time-series experiments,
and direct potential-outcome formulations for observational time
series \citep{Robins1994, QianYooKlasnjaAlmirallMurphy2021,
BojinovShephard2019, RambachanShephard2021}. They are recovered by
\[
\mathfrak{D}_{F_t}(\mu,\nu)
\;:=\;\int F_t\,d\mu-\int F_t\,d\nu,
\qquad F_t(x):=f_t(x(t)).
\]
\emph{Integrated effects} appearing in history-restricted and
history-adjusted marginal structural models, dynamic treatment
regimes, and sustained-strategy formulations
\citep{NeugebauerVanDerLaanJoffeTager2007,
PetersenDeeksMartinVanDerLaan2007, Murphy2003, HernanRobins2016} are
recovered by
\[
\mathfrak{D}_{F_{\mathrm{int}}}(\mu,\nu)
\;:=\;\int F_{\mathrm{int}}\,d\mu-\int F_{\mathrm{int}}\,d\nu,
\qquad F_{\mathrm{int}}(x):=\int_0^T f_t(x(t))\,dt.
\]
\emph{Genuinely path-level effects} — central to continuous-time
g-computation and continuous-time marginal structural models
\citep{GillRobins2001, Roysland2011}, local-independence and dynamic
path-analysis \citep{Didelez2008, AalenRoyslandGranLedergerber2012},
dynamic structural causal models and causal interpretations of
stochastic differential equations
\citep{RubensteinBongersMooijScholkopf2018, BoekenMooij2024,
SokolHansen2014}, and recent functional-longitudinal causal inference
\citep{Ying2024, SunCrawford2022} — are recovered by
\[
\mathfrak{D}_\Phi(\mu,\nu)
\;:=\;\int \Phi\,d\mu-\int \Phi\,d\nu,
\qquad \Phi:C([0,T],\Gamma_{X_i})\to\mathbb{R},
\]
or, more generally, by any law-level discrepancy on path measures
(e.g., KL divergence). 

\looseness=-1A separate, adjacent family studies temporal influence through
\emph{Granger-style} or \emph{information-flow} quantities, including
Granger-causal predictability \citep{WhiteLu2010,
BarnettSeth2015}, transfer entropy \citep{Schreiber2000,
BarnettBarrettSeth2009}, directed information \citep{Massey1990,
Kramer1998, AmblardMichel2013}, graphical or structural time-series
causality \citep{EichlerDidelez2007, PetersJanzingScholkopf2013,
RungeGerhardusVarandoEtAl2023}, and causation-entropy criteria
\citep{SunBollt2014, SunTaylorBollt2015, SunCafaroBollt2014,
SurasingheBollt2020, RungeBathianyBolltEtAl2019}. These quantities
measure directed predictability rather than
intervention effect, and are therefore outside the scope of
Definitions~\ref{def:hcm-causal-effect} and~\ref{def:hcm-ate}. They
are nonetheless central to our purposes: as we show in the next
section, the thermodynamic quantities along an HCM trajectory provide
a natural account of such information-flow measures
\citep{ProkopenkoLizierPrice2013}, bridging the interventional and
information-flow perspectives within a single framework.


\section{The Statistical Footprint of Causality in Non-equilibrium Thermodynamics}
\label{sec:thermo}

\looseness=-1We now pass from the definition of HCMs to the thermodynamic
quantities they naturally generate and their relationship with the HCM causal structure. For every intervention policy and
every induced trajectory, the model assigns observables with a direct
physical meaning. These quantities provide a measurable description of
the effect of interventions on a system, and they will allow us to
connect the Hamiltonian structure introduced in Section~\ref{sec:hcm}
to causal notions such as parenthood and causal effect.

\paragraph{Thermodynamic preliminaries.}
Fix an admissible intervention policy $\lambda\in\mathcal{A}$, and let
$z(\cdot)$ be a microscopic trajectory generated by the controlled
Hamiltonian $H_\lambda(t,z):=H(z,\lambda_t)$.

\begin{definition}[Work]
\label{def:work}
The \emph{work} performed by the intervention policy $\lambda$ along
the trajectory $z(\cdot)$ is
\[
W(z;\lambda)
\;:=\;
\int_0^T \partial_t H_\lambda(t,z(t))\,dt
\;=\;
\int_0^T
\left\langle
\partial_\lambda H(z(t),\lambda(t)),\,\partial_t\lambda(t)
\right\rangle dt,
\]

where the second equality holds whenever $\lambda$ is differentiable and $\langle \cdot, \cdot \rangle$ denotes the standard scalar product.
\end{definition}

\begin{definition}[Heat]
\label{def:heat}
Let $z(t)=(x(t),e(t))$ be a controlled HCM trajectory and $e_i(t)$
the bath attached to node $X_i$. The \emph{heat} transferred from
bath $E_i$ to the observed system over $[0,T]$ is
\[
Q_i[z]
\;:=\;
-\bigl(H_{E_i}(e_i(T))-H_{E_i}(e_i(0))\bigr).
\]
\end{definition}

With this sign convention, $Q_i[z]>0$ means net energy flows from
bath $E_i$ into the observed system.

\paragraph{Phenomenological causality and work.}
The thermodynamic quantities introduced above connect to several
aspects of causality. We begin with the most immediate one: under the
Phenomenological Causality stance \citep{janzing2022phenomenologicalcausality} described in
Section~\ref{sec:hcm}, causal effect is whatever an intervention can
produce, and \emph{work} is the natural measure of the energy footprint of an
intervention on a system (note in particular that, if an intervention $\lambda$ is constant in time, then $W[z;\lambda]=0$). The following result formalizes this:
intervention-induced work detects which variables enter the mechanism
of a given node, providing a thermodynamic characterization of
mechanism dependence.

\begin{proposition}[Work rate identifies mechanism dependence]
\label{prop:work-rate-parents}
For a given HCM and a node \(i\), assume the admissible interventions class 
\(\mathcal A\) contains a \(\lambda\) that is differentiable at a fixed time $t^* \in (0,T)$ and such that $
\partial_t \lambda_i(t^*) \neq 0$ and $\partial_t \lambda_k(t^*)=0$ for $k\neq i$, 
which is locally faithful at \(\lambda(t^*)\) in the sense of
Assumption~\ref{ass:local-work-rate-faithfulness}. Define the work rate of $\lambda$ at $t^*$ by $
\dot W (t, z; \lambda)
:=
\left\langle
\partial_\lambda H(z,\lambda(t)),\partial_t\lambda(t)
\right\rangle.
$ Then, for every \(j\neq i\), $j \in \mathcal{D}_i$
 if and only if there exists a state $z^*$ such that:
\[
 \left[ \nabla_x \dot W(t^*, z^*; \lambda)\right]_j \neq 0 .
\]
\end{proposition}

\looseness=-1This proposition is appealing because work directly tracks which variables
participate in a local mechanism. Its limitation is practical: computing
\(\dot W\) requires explicit access to the Hamiltonian, just as computing a
mechanism-level causal effect in an SCM would require access to the structural
equations. We therefore turn to a thermodynamic quantity that can be
estimated from trajectories: \textit{entropy production}.

The idea behind entropy production is to separate genuine temporal irreversibility from the
effect due to the state and the intervention protocol changing over time. An intervened
trajectory is compared to the corresponding
reversed experiment: the microscopic state $z_t$ is time-reversed and the control
schedule $\lambda_t$ is run backward. Entropy production is the log-likelihood asymmetry
between these two ensembles. In this sense, it removes the trivial
asymmetry due to the changing controls and the state responding to them and asks: how statistically
distinguishable is the forward experiment from the same experiment recorded
backward in time? This residual arrow-of-time signal is the thermodynamic
footprint of performing $\lambda$ on a given state, and below we show that its local forms
retain causal information while being accessible from trajectory data
\citep{Otsubo_2020, Otsubo_2022}.


\begin{definition}[Backward path law]
\label{def:time-reversal}
Let $\Theta$ be the time-reversal operator from
Definition~\ref{def:hcm}. Given
$\lambda\in\mathcal{A}$ define the
\emph{time-reversed protocol} by
$(\Theta\lambda)(t):=\lambda(T-t)$
and pathwise time reversal $\mathcal{R}:\mathcal{Z}\to\mathcal{Z}$ by
$(\mathcal{R}z)(t):=\Theta z(T-t)$, where
$\mathcal Z:=C([0,T],\Gamma)$ is the microscopic system-bath
path space. Assume the forward initial law is
$\rho_0\otimes\gamma_E$.
Let $\rho_T^\lambda$ be the forward final law of $X$ and
$\Theta_X$ the restriction of $\Theta$ to $X$.
Let $\mathbb Q_R^\lambda$ denote the microscopic path law
generated under $\Theta\lambda$ from the initial law
$(\Theta_X)^\#\rho_T^\lambda\otimes\gamma_E$.
The \emph{backward observed path law} is $\mathbb P_B^{\Theta\lambda}
:=
(\Pi_X\circ\mathcal R)^\#\mathbb Q_R^\lambda.$
\end{definition}

\begin{definition}[Pathwise entropy production]
\label{def:global-ep}
Assume $\mathbb{P}_F^\lambda$ is absolutely continuous with respect
to $\mathbb{P}_B^{\Theta\lambda}$.
The \emph{pathwise entropy production} of the controlled process is: $
\Sigma^\lambda(x)
:=
\log
\frac{d\mathbb P_F^\lambda}
{d\mathbb P_B^{\Theta\lambda}}(x).$
Its expectation $
\mathbb E_{\mathbb P_F^\lambda}[\Sigma^\lambda]
=
D_{\mathrm{KL}}\!\left(
\mathbb P_F^\lambda
\,\middle\|\,
\mathbb P_B^{\Theta\lambda}
\right)$ is called the \textit{mean entropy production}.
\end{definition}

Unlike Shannon entropy, which is a state function of the instantaneous law,
\(\Sigma^\lambda\) is a path-space quantity: it distinguishes processes with the
same endpoint distributions but different transient evolution. Under the assumptions stated in Appendix~\ref{app:hcm}, this
pathwise asymmetry coincides with the total entropy change of the observed system
and its baths

\begin{equation}
\label{eq:global_q_balance}
\Sigma^\lambda
\;=\;
\Delta s_{\mathrm{sys}}^\lambda
\;-\;
\sum_i \beta_i Q_i^\lambda,
\end{equation}

where $\Delta s_{\mathrm{sys}}^\lambda
=-\log p_T^\lambda(x(T))+\log p_0^\lambda(x(0))$ is the standard
entropy change of the observed system under $\lambda$. Equipped with this interpretation, the inequality  $
\mathbb E_{\mathbb P_F^\lambda}[\Sigma^\lambda] \geq 0$ that follows from $\mathbb E_{\mathbb P_F^\lambda}[\Sigma^\lambda]
=
D_{\mathrm{KL}}\!\left(
\mathbb P_F^\lambda
\,\middle\|\,
\mathbb P_B^{\Theta\lambda}
\right)$ is the derivation of the \textit{Second Law of Thermodynamics} for Hamiltonian systems  \citep{Jarzynski1997NEQ}. Thus, entropy production is both a
statistical measure of time-reversal asymmetry and a physical measure of
dissipation.

The quantities introduced here are standard objects of stochastic and
non-equilibrium thermodynamics. Three traditions feed into the
formulation we adopt: trajectory-level work and heat in stochastic
energetics \citep{Sekimoto1998Langevin, Sekimoto2010}, fluctuation
theorems and entropy production as a path-likelihood ratio
\citep{Jarzynski1997NEQ, Crooks1999FT, Seifert2012Review}, and
information thermodynamics for coupled subsystems
\citep{SagawaUeda2012Review, ItoSagawa2013CausalNetworks,
ParrondoHorowitzSagawa2015Review}, whose tradition expands well beyond the works cited here. Our aim here is
not to redefine these objects, but to localize them at the node level
inside an explicitly intervention-based causal model, so that entropy
production becomes a witness of causal effect.

\textbf{Projected local entropy production. \ }
The entropy production \(\Sigma^\lambda\) above is global: it compares the
forward and backward laws of the whole observed trajectory. Causal effect,
however, was defined in Section~\ref{sec:hcm} through changes in the marginal
path law of a single variable. The first localization of entropy production is
therefore obtained by projecting the path-space asymmetry onto the trajectory of
one node. For a node \(i\), let $
\mathbb P_i^\lambda
:=
(\Pi_i)_\#\mathbb P_F^\lambda$
be the marginal observed path law of \(X_i(\cdot)\), and let
\(\mathcal R_i\) denote time reversal on \(X_i\)-paths. The projected entropy
production of node \(i\) is the time-reversal asymmetry visible from observing
only \(X_i\).

\begin{definition}[Projected local entropy production]
\label{def:projected-local-ep}

The \emph{projected local entropy production} of node \(i\) under the appropriate absolute continuity assumption for the marginal measures is
\[
\Sigma_i^{\mathrm{proj},\lambda}(x_i)
:=
\log
\frac{
\mathrm{d}\mathbb P_i^\lambda
}{
\mathrm{d}(\mathcal R_i)_\#\mathbb P_i^\lambda
}(x_i).
\]

Its mean with respect to $\mathbb{P}_i^\lambda$ is
$
\mathbb E_{\mathbb P_i^\lambda}[\Sigma_i^{\mathrm{proj}, \lambda}(x_i)]
:=
D_{\mathrm{KL}}\!\left(
\mathbb P_i^\lambda
\,\middle\|\,
(\mathcal R_i)_\#\mathbb P_i^\lambda
\right).
$
\end{definition}

Projected entropy production is local in exactly the same sense as our causal
effect definition: it is a functional of the marginal path law of \(X_i\).
Consequently, changes in projected entropy production provide a one-way witness
of causal effect.

\begin{tcolorbox}[
  colback=InterventionGreen!10,
  colframe=InterventionGreen!75!black,
  fonttitle=\bfseries\small,
  sharp corners,
  boxrule=0.9pt,
  left=8pt,
  right=8pt,
  top=5pt,
  bottom=5pt,
]
\begin{proposition}[Projected entropy production witnesses causal effect]
\label{thm:proj-ep-witness}
Fix $i\neq j$ and a baseline policy $\lambda_{-j}$. For $k=1,2$, let
$\lambda^{(k)}:=(\lambda_j^{(k)},\lambda_{-j})\in\mathcal{A}$, and
assume $\Sigma_i^{\mathrm{proj},\lambda^{(k)}}$ is well-defined and integrable. If
\[
\mathbb E_{\mathbb P_i^{\lambda^{(1)}}}
\left[\Sigma_i^{\mathrm{proj}, \lambda^{(1)}}\right]
\;\neq\;
\mathbb E_{\mathbb P_i^{\lambda^{(2)}}}
\left[\Sigma_i^{\mathrm{proj}, \lambda^{(2)}}\right],
\]
then $X_j$ has a causal effect on $X_i$ over $[0,T]$ in the sense of
Definition~\ref{def:hcm-causal-effect}.
\end{proposition}
\end{tcolorbox}

Proposition~\ref{thm:proj-ep-witness} establishes projected entropy
production as a sufficient witness of causal effect in the sense of
Definition~\ref{def:hcm-causal-effect}:
in the Langevin
specialization experiments, we use current-restricted NEEP
\citep{Kim_2020,Otsubo_2020} as a practical proxy for projected entropy production: further exploration of statistical estimators for this quantity is left to future work. In this sense, projected entropy production detects
whether \(j\) has an effect on \(i\) in the sense of \ref{def:hcm-causal-effect}, but not yet whether \(j\) is a paired node. This is consistent with several results in causality, where soft interventions identify ancestral relations in DAGs~\citep{tian2001causal} and directed connectivity in cyclic models~\citep{mooij2013cyclic}. 

\looseness=-1\textbf{Detecting Mechanism Dependence in Langevin Dynamics.}
We next study a complementary local entropy-production
response that identifies mechanism dependence in the
controlled Langevin setting of   Theorem~\ref{thm:realization-informal}. In this setting the probability evolution follows
the Fokker-Planck equation: $
\partial_t p^\lambda(x,t)
=
-\sum_{k=1}^n \partial_{x_k}J_k^\lambda(x,t),
$
where \(J_k^\lambda\) is the probability current associated with coordinate
\(X_k\). Inspired by the energetic characterization of global entropy production in
\eqref{eq:global_q_balance}, and following the information-flow formulation of
subsystem thermodynamics \citep{HorowitzEsposito2014InfoFlow, ItoSagawa2013CausalNetworks}, we define the local entropy production of node \(i\) as
the marginal entropy balance of \(X_i\), corrected by the heat exchanged with its
bath and by the information flow between \(X_i\) and the rest of the system.

\begin{definition}[Local entropy production]
\label{def:local-ep}
Fix node \(i\) with bath \(E_i\) and let $
s_i^{\mathrm{marg},\lambda}(t)
:=
-\int p_{i,t}^\lambda(x_i)\log p_{i,t}^\lambda(x_i)\,dx_i$
be the marginal Shannon entropy of \(X_i(t)\), and let
\(Q_i^\lambda\) denote the heat exchanged between bath \(E_i\) into the
observed system over \([0,T]\). Denote the contribution of the \(i\)-th current to the joint system entropy rate by $
\dot s_{\mathrm{sys}}^{(i),\lambda}(t)
:=
-\int_{\mathbb R^n}
J_i^\lambda(x,t)\,
\partial_{x_i}\log p^\lambda(x,t)\,dx$.
We define the information-flow rate associated with node \(i\) by:

\vspace{-30pt}

\[
\dot I_i^{\mathrm{flow},\lambda}(t)
:=
\frac{d}{dt}\left( s_i^{\mathrm{marg},\lambda}(t) \right)
-
 \dot{s}_{\mathrm{sys}}^{(i),\lambda}(t).
\]

\vspace{-30pt}

The \emph{local entropy production} of node \(X_i\) over \([0,T]\) is

\vspace{-30pt}

\[
\Sigma_i^{\mathrm{loc},\lambda}
:=
\Delta s_i^{\mathrm{marg}, \lambda}
-\beta_i \mathbb{E}_\lambda[Q_i^\lambda]
-
\int_0^T
\dot I_i^{\mathrm{flow},\lambda}(t)\,dt .
\]
\end{definition}

\vspace{-25pt}

\(\Sigma_i^{\mathrm{loc}}\) is still a trajectory-based object: in the
Langevin case it can be estimated from samples of the full observed
trajectory \(x(\cdot)\) using variational estimators
\citep{Otsubo_2020,Otsubo_2022}. The price for moving from projected to
channel-local entropy production is therefore additional observational
information: projected EP requires only \(x_i(\cdot)\), whereas
\(\dot\Sigma_i^{\mathrm{loc}}\) uses the full state as context. The gain is
structural: while projected EP witnesses finite-time causal effects,
the response of the channel-local EP rate to a right-local intervention identifies mechanism dependence.

\begin{tcolorbox}[
  colback=InterventionGreen!10,
  colframe=InterventionGreen!75!black,
  fonttitle=\bfseries\small,
  sharp corners,
  boxrule=0.9pt,
  left=8pt,
  right=8pt,
  top=5pt,
  bottom=5pt,
]
\begin{theorem}[Mechanism dependence from local entropy production]
\label{thm:discovery-main}
\looseness=-1
Consider a controlled Langevin diffusion of the kind described in
Theorem~\ref{thm:realization-informal}. Fix an intervention time
\(t^*\in(0,T)\), and write
\[
p_*:=p(\cdot,t^*),
\qquad
\lambda_*:=\lambda(t^*).
\]
Assume the admissible policy class \(\mathcal A\) contains two right-local
protocols initialized at \((p_*,\lambda_*)\): a baseline protocol
\(\lambda^0\) with \(\lambda^0(s)=\lambda_*\), and a \(j\)-probe protocol
\(\lambda^v\) with
\[
\lambda^v(0)=\lambda_*,
\qquad
\lambda^v(s)=\lambda_*+vse_j+o(s)
\quad\text{as }s\downarrow0,
\]
where \(v\in\mathbb R\) is the intervention velocity. Let
\(p^0(\cdot,s)\) and \(p^v(\cdot,s)\) be the corresponding laws, both
initialized at \(p_*\). For \(i\neq j\), define the local
entropy-production response
\[
R_{i\mid j}(t^*)
:=
\lim_{v\to0}\lim_{s\downarrow0}
\frac{
\dot\Sigma_i^{\mathrm{loc}}[p^v(\cdot,s),\lambda^v(s)]
-
\dot\Sigma_i^{\mathrm{loc}}[p^0(\cdot,s),\lambda^0(s)]
}{vs},
\]
where \(\dot\Sigma_i^{\mathrm{loc}}[p,\lambda]\) denotes the local
entropy-production rate functional. Then, under
Assumption~\ref{ass:discovery-compact},
\[
\boxed{
i\in\mathcal D_j
\quad\Longleftrightarrow\quad
R_{i\mid j}(t^*)\neq0.
}
\]
\end{theorem}
\end{tcolorbox}

\vspace{3pt}

\begin{remark}[Schwarz symmetry and the role of intervention]
\label{rem:schwarz}
Uncontrolled gradient dynamics has $\partial_{x_i}\partial_{x_j} U = \partial_{x_j}\partial_{x_i} U$,
so observational data alone identifies only an undirected dependency structure.
Soft interventions break this symmetry through the asymmetric mixed partials
$\partial_{x_i}\partial_{\lambda_j} U$, which is precisely what $R_{i \mid j}$ probes.
\end{remark}

\section{Experimental Results}
\label{sec:experimental_results}

\looseness=-1The goal of these experiments is to demonstrate that entropy production is
not just a theoretical descriptor of causal models, but can be
\emph{measured} from trajectories and used as a working causal estimand.

The experiments investigate two uses of entropy-production
estimators: comparing intervention protocols through trajectory-level
proxies, and recovering mechanism dependence through local
entropy-production responses.

\textbf{ATE experiments: validating Proposition~\ref{thm:proj-ep-witness}. \ }
%
In light of Proposition~\ref{thm:proj-ep-witness} we use the projected entropy
production $\mathbb E_{\mathbb P_i^{\lambda^{(1)}}}\left[\Sigma_i^{\mathrm{proj}, \lambda^{(1)}}\right]$ as a measure of path-level causal effect,
estimated with the current-restricted variant of the NEEP estimator ~\citep{Kim_2020, Otsubo_2020}, with the interpretation specified in Remark~\ref{rem:projected-ep-estimation}.
We consider two qualitatively distinct settings, each designed so that a
standard observable-level estimand is zero by construction even though
the two protocols drive the system along visibly different paths:
\begin{itemize}
  \item[\textbf{(A)}] \textbf{Same endpoints, different paths.}
    A linear chain of overdamped oscillators is driven on its root node
    by two protocols that reach the same final value but follow
    different trajectories (a smooth ramp vs.\ an excursion that
    overshoots and returns).
    Because both trajectories have the same terminal values, the 
    \emph{endpoint} ATE on the driven node is zero by design.
  \item[\textbf{(B)}] \textbf{ Same interventions, reversed in time.}
    A small system with a hidden nonlinear mediator receives the same
    two pulses applied in opposite order.
    On a symmetric observable summing all coordinates, a
    time-translation argument forces the \emph{cumulative} ATE
    $cATE(Y)=\int_0^T\!\mathbb{E}[Y_t\mid \mathrm{do}(A \circ B)]-\mathbb{E}[Y_t\mid \mathrm{do}(B \circ A)]\,dt$ to
    vanish.
\end{itemize}
\looseness=-1We compare the \textit{projected EP contrast} $ATE(\Sigma_i^\mathrm{proj}):=\mathbb E_{\mathbb P_i^{\lambda^{(1)}}}\left[\Sigma_i^{\mathrm{proj}, \lambda^{(1)}}\right] - \mathbb E_{\mathbb P_i^{\lambda^{(2)}}}\left[\Sigma_i^{\mathrm{proj}, \lambda^{(2)}}\right]$
against \emph{endpoint} ATE in case (A) and \emph{cumulative} ATE in case (B), with
bootstrap-generated $95\%$ confidence intervals on the ATEs and a one-sample $t$-test for the projected entropy production
across independent seeds for NEEP.
Table \ref{tab:ep_vs_ate} reports the headline numbers: in both settings
the observable ATE that the design forces to vanish is statistically
indistinguishable from $0$, while entropy production separates the two protocols confidently.
This shows empirically that the projected entropy production is able to detect path-level causal effects in the sense of Definition \ref{def:hcm-causal-effect} which are invisible to common versions of the ATE. While other path-level discrepancies could also detect such differences, we emphasize that entropy production provides a physically meaningful and estimable discrepancy tied to irreversibility. 

\begin{table}[h]
\centering

\vspace{2pt}
\small
\begin{tabular}{@{}lll@{}}
\toprule
Setting & Designed-zero observable ATE (95\% CI) & $ATE(\Sigma_i^\mathrm{proj})$ (NEEP, mean$\pm$SD, $p$) \\
\midrule
(A) Same endpoints       & Endpoint ATE $=-0.012$\, $[-0.037,+0.012]$
                         & $-0.366\pm0.001$,\ $p=2.3\!\times\!10^{-6}$ \\
(B) Reversed in time     & Cumulative ATE $=+0.002$\, $[-0.062,+0.067]$
                         & $-0.417\pm0.017$,\ $p=6.8\!\times\!10^{-7}$ \\
\bottomrule
\end{tabular}
\vspace{2mm}
\caption{Marginal/temporal ATE is zero by construction, so specific choices like Endpoint and Cumulative ATE miss this causal effect;
  $ATE(\Sigma_i^\mathrm{proj})$ resolves the path-level effect and identifies a non-zero causal effect.}
\label{tab:ep_vs_ate}
  \vspace{-2em}
\end{table}

\vspace{3pt}

\textbf{Discovery experiments: validating Theorem~\ref{thm:discovery-main}. \ }
%
To test the mechanism dependence identification condition induced by Theorem \ref{thm:discovery-main} we generate random cyclic
Erd\H{o}s--R\'enyi directed graphs ($n\!=\!15$, edge probability $0.10$) and instantiate three families of overdamped Langevin
HCMs with potentials: 
$U(x,\lambda) = \tfrac12\sum_i a_i x_i^2 + U_{\rm self}(x)
+ \sum_j \lambda_j[b_j x_j + \sum_{i\in \mathcal{D}_j} c_{ij}\,\phi(x_i,x_j)]$
shared $\sigma_0\!=\!0.5$
and the following kernels $\phi$: {Bilinear} ($U_{\rm self}\!=\!0,\ \phi=x_i x_j$); {Sigmoid} ($U_{\rm self}\!=\!0,\ \phi=\tanh(\beta x_i)\,x_j$,
  $\beta\!=\!2.5$); {Cubic} ($U_{\rm self}\!=\!\tfrac{\gamma}{4}\sum_i x_i^4,\
  \phi=x_i x_j$, $\gamma\!=\!0.5$).
\looseness=-1For each graph we pick a target $j$ randomly and estimate $R_{i|j}$ as a finite difference of $\dot\Sigma_i^{\mathrm{loc}}$ estimated with the full-system version of the NEEP estimator \citep{Otsubo_2020}, fixing $\lambda_* = 2$, $\Delta t\!=\!0.02$ and $N=8000$. Full configurations for NEEP's training and the hyperparameters of the potential $U$ are provided in Appendix \ref{app:experiments}.
The mechanism dependence set is selected by thresholding values of the estimated $R_{i|j}$ via $\widehat{\mathcal{D}}_j = \{i:|\widehat R_{i\mid j}|>0.10\cdot\max_i|\widehat R_{i\mid j}|\}$.
We report mean F1, precision, recall, and local SHD
$\bigl| \widehat{\mathcal{D}}_j\,\triangle \mathcal{D}_j) \bigr|$ across $10$ random graphs
per potential, along with 95\% confidence intervals in table \ref{tab:hcm_discovery_main}. We observe near-perfect recall paired with strong precision in the recovery results.

\begin{table}[h]
\centering

\vspace{2pt}
\small
\begin{tabular}{@{}lcccc@{}}
\toprule
Potential & F1 & Precision & Recall & local SHD $\downarrow$ \\
\midrule
Bilinear & $0.839_{\,[0.77,\,0.91]}$ & $0.775_{\,[0.66,\,0.90]}$
         & $0.967_{\,[0.90,\,1.00]}$ & $1.10_{\,[0.5,\,1.8]}$ \\
Sigmoid  & $0.858_{\,[0.79,\,0.93]}$ & $0.800_{\,[0.69,\,0.91]}$
         & $0.955_{\,[0.89,\,1.00]}$ & $0.90_{\,[0.4,\,1.5]}$ \\
Cubic    & $0.825_{\,[0.70,\,0.93]}$ & $0.745_{\,[0.59,\,0.89]}$
         & $0.971_{\,[0.91,\,1.00]}$ & $1.40_{\,[0.4,\,2.7]}$ \\
\bottomrule
\end{tabular}
\vspace{2mm}
\caption{Recovery on random cyclic ER graphs ($n\!=\!15$, $10$ graphs
  per row) using the results of Theorem~\ref{thm:discovery-main}. All metrics with $95\%$ bootstrap CIs. Performance is very strong despite the challenging non-stationary and cyclic setting.
  }
\label{tab:hcm_discovery_main}
\vspace{-1em}
\end{table}



\vspace{-2mm}
\section{Conclusions}
\vspace{-2mm}
In this paper, we proposed an alternative language for statistical causality that is closely related to how causation is described in physical systems~\citep{rovelli2023oriented}. This is appealing because concepts that are traditionally difficult to model with existing frameworks (a notable example is the definition of interventional distributions in cyclic SCMs~\citep{bongers2021foundations}) become natural. In simple numerical experiments, we have empirically demonstrated that entropy production is an effective witness of causal effects that standard treatment effect formulations struggle to capture and that local entropy-production responses support
mechanism-dependence recovery. Clearly, the SCM framework is flexible, and several phenomena described in our paper can also be expressed ad-hoc in that language. At the same time, we have shown in Theorem \ref{thm:realization-informal} that HCMs generalize SCMs and more naturally model the temporal phenomena arising in physical systems.

In this new language, we hope that many of the concepts already existing in causality will be re-discovered, now capable of modeling broader classes of physical systems. A notable first example is the notion of counterfactuals, which can be naturally posed by noise abduction to the bath trajectories. However, how to perform this abduction step is non-trivial as the noise is no longer just a distribution but a probability flow. Secondly, it must be stated that despite describing many processes such as Langevin diffusions, general temporal SCMs \citep{PetersJanzingScholkopf2013} with non-gradient drifts are not immediately described by HCMs, and extending the present theory to said framework is a prominent next step we reserve for the future. Finally, the learning of the Hamiltonians and its connections to causal discovery. Classically, causal discovery from observational data rests on structural assumptions such as non-linear additive noise models, and we currently do not know what the HCM analogue is. At the same time, we highlight an interesting connection with the score-matching literature in causal discovery~\citep{rolland2022score,montagna2023causal}, which appears in our models in the continuity equation of Langevin processes. The intersection with causal representation learning~\citep{scholkopf2021toward}, perhaps via diffusion models, is also an interesting, open-ended future direction.

\section*{Acknowledgments and Funding}

The authors thank Sosuke Ito, Dominik Janzing, Sara Magliacane and Joris Mooij for providing useful insights on the development of the process. DR is financed by the project “Building Energy Systems on causal reasoning (BOSS)“, funded within the “Technologies and Innovations for the Climate-Neutral City” (TIKS) Programme of the Austrian Research Promotion Agency (FFG).

\bibliographystyle{plainnat}
\bibliography{refs-hcm-v2}


\newpage

\appendix

\section{Foundations of Hamiltonian Causal Models}
\label{app:hcm}

In this section we expand the context on Definition~\ref{def:hcm}. We discuss precisely the role of our assumptions and give context to related assumptions in similar settings in both Physics and Causality.

\begin{definition}[Hamiltonian Causal Model, detailed]
\label{def:hcm-detailed}
Fix a time horizon \(T\in\mathbb R_+\cup\{+\infty\}\). A
\emph{Hamiltonian Causal Model} is a tuple
\[
\mathcal M
=
(\mathcal P,z_0,\Gamma,\omega,\mathcal L,\mathcal A,H)
\]
with the following components.

\begin{enumerate}[label=\textnormal{(\arabic*)},leftmargin=1.7em,itemsep=0.35em]

\item \textbf{Probability space and initial condition.}
\(\mathcal P=(\Omega,\mathcal F,\mathbb P)\) is a probability space and
\(z_0:\Omega\to\Gamma\) is a random initial state.

\item \textbf{Phase space.}
\((\Gamma,\omega)\) is a symplectic manifold with node-wise product
factorization
\[
\Gamma=\Gamma_X\times\Gamma_E,
\qquad
\Gamma_X=\prod_{i=1}^n\Gamma_{X_i},
\qquad
\Gamma_E=\prod_{i=1}^n\Gamma_{E_i}.
\]
The coordinates \(X_i\) are the observed system variables, while \(E_i\) is the
unobserved local environment, or bath, attached to \(X_i\).

\item \textbf{Intervention space and admissible policies.}
The intervention manifold factorizes as
\[
\mathcal L=\prod_{i=1}^n\mathcal L_i .
\]
We fix an information filtration
\((\mathcal G_t)_{t\in[0,T]}\), representing the information available to the
operator up to time \(t\). The admissible class \(\mathcal A\) consists of
nonanticipative policies
\[
\lambda:[0,T]\times\Omega\to\mathcal L,
\qquad
\lambda_t=(\lambda_{1,t},\ldots,\lambda_{n,t}),
\]
such that \(\lambda_t\) is \(\mathcal G_t\)-measurable for every \(t\).
Deterministic open-loop protocols are included as the special case in which
\(\lambda_t\) is non-random. When a result involves work rates,
time-reversal protocols, or infinitesimal probes, we restrict to the relevant
subclass of policies whose sample paths are sufficiently regular, typically
piecewise \(C^1\) or deterministic \(C^1\).

\item \textbf{Hamiltonian and mechanism decomposition.}
The Hamiltonian
\[
H:\Gamma\times\mathcal L\to\mathbb R
\]
is smooth and decomposes as
\begin{equation}
\label{eq:app-hcm-decomposition}
H(z,\lambda)
=
\sum_{i=1}^n H_{E_i}(e_i)
+
\sum_{i=1}^n
\left[
H_i^{\mathrm{self}}(x_i,\lambda_i)
+
H_i^{\mathrm{mech}}(x_i,x_{-i},\lambda_i)
+
H_i^{\mathrm{bath}}(x_i,e_i,\lambda_i)
\right].
\end{equation}
Here \(H_i^{\mathrm{self}}\) is the intrinsic term of node \(i\),
\(H_i^{\mathrm{mech}}\) encodes the influence of other observed variables on
\(X_i\), and \(H_i^{\mathrm{bath}}\) encodes the interaction between \(X_i\)
and its local bath \(E_i\). The intervention coordinate \(\lambda_i\) acts only
on the node-\(i\) mechanism and its bath coupling.

\item \textbf{Hamiltonian dynamics.}
For a policy \(\lambda\in\mathcal A\), define the time-dependent Hamiltonian
\[
H_\lambda(t,z):=H(z,\lambda_t).
\]
The controlled Hamiltonian vector field \(Y_{H_\lambda}\) is defined by
\[
\omega(Y_{H_\lambda}(t,\cdot),\cdot)
=
\mathrm d H_\lambda(t,\cdot).
\]
The corresponding controlled flow is denoted by
\(\Phi_{s,t}^{\lambda}:\Gamma\to\Gamma\), and
\(\Phi_t^\lambda:=\Phi_{0,t}^\lambda\). The microscopic path map is
\[
\Psi^\lambda(\omega)(t)
:=
\Phi_t^\lambda(z_0(\omega)).
\]
The forward observed path law is
\[
\mathbb P_F^\lambda
:=
(\Pi_X\circ\Psi^\lambda)^\#\mathbb P,
\]
where \(\Pi_X\) projects a full system--bath trajectory onto its observed
coordinates.

\end{enumerate}
\end{definition}

We recall the definition of mechanism dependence:
\begin{definition}[Mechanism dependence]
\label{def:app-hcm-parents}
As in Definition~\ref{def:hcm-adjacent}, define
\[
\mathcal D_i
:=\{j\neq i:
\partial_{x_j}H_i^{\mathrm{mech}}\not\equiv0\}.
\]
\end{definition}

\begin{assumption}[Standing HCM hypotheses]
\label{app:HCM-hypotheses}
Unless stated otherwise, the following hypotheses are assumed.

\begin{enumerate}[label=\textnormal{(H\arabic*)},leftmargin=1.7em,itemsep=0.35em]

\item \textbf{Initial bath factorization.}
The initial environment law factorizes across nodes:
\[
\mu_E^0
:=
(\pi_E\circ z_0)^\#\mathbb P
=
\bigotimes_{i=1}^n \mu_{E_i}^0 .
\]
Thus the baths provide independent exogenous randomness at the initial time.

\item \textbf{Reservoir initialization and local equilibrium.}
Each bath is initialized at inverse temperature \(\beta_i>0\), typically in the
Gibbs state
\[
\mu_{E_i}^0(\mathrm d e_i)
=
Z_i^{-1}
\exp\!\bigl(-\beta_i H_{E_i}(e_i)\bigr)\,\mathrm d e_i .
\]
We work in the standard reservoir regime in which the bath remains an
equilibrium noise source at the thermodynamic level relevant for entropy
bookkeeping. Equivalently, heat exchanged with bath \(E_i\) satisfies the local
detailed-balance convention
\[
\Delta s_{E_i}
=
-\beta_i Q_i,
\qquad
Q_i
:=
-\bigl(H_{E_i}(e_i(T))-H_{E_i}(e_i(0))\bigr).
\]
For finite Hamiltonian baths this is an idealization; it is recovered in
weak-coupling, large-reservoir, or thermodynamic-limit regimes.

\item \textbf{Microscopic reversibility.}
There exists an anti-symplectic involution
\(\Theta:\Gamma\to\Gamma\) such that
\[
\Theta^2=\mathrm{id},
\qquad
\Theta^*\omega=-\omega,
\]
and \(\Theta\) does not mix observed and environmental coordinates:
\[
\pi_X\circ\Theta=\Theta_X\circ\pi_X,
\qquad
\pi_E\circ\Theta=\Theta_E\circ\pi_E .
\]
For each fixed control value,
\[
H(\Theta z,\lambda)=H(z,\lambda).
\]
Writing the time-reversed protocol as
\[
(\Theta\lambda)(t):=\lambda(T-t),
\]
the controlled flows satisfy
\[
\Theta\circ \Phi_{s,t}^{\lambda}
=
\Phi_{T-s,T-t}^{\Theta\lambda}\circ\Theta .
\]

This is the microscopic reversibility condition underlying the
forward/backward path-law comparison used to define entropy production.

\item \textbf{Well-posed controlled evolution.}
For every \(\lambda\in\mathcal A\), the controlled Hamiltonian equation
\[
\dot z(t)=Y_{H(\cdot,\lambda_t)}(z(t)),
\qquad z(0)=z_0,
\]
admits a unique global nonanticipative solution on \([0,T]\). The induced
microscopic path map is denoted by \(\Psi^\lambda\). For deterministic
open-loop protocols, this solution is generated by a controlled Hamiltonian
flow \(\Phi_{s,t}^{\lambda}:\Gamma\to\Gamma\), so that
\[
\Psi^\lambda(\omega)(t)=\Phi_{0,t}^{\lambda}(z_0(\omega)).
\]

\item \textbf{Interventional faithfulness.}
The intervention coordinates are faithful to the mechanism graph: for every
\(i\neq j\),
\[
j \in \mathcal{D}_i
\quad\Longleftrightarrow\quad
\partial_{x_j}\partial_{\lambda_i}H(z,\lambda)\not\equiv0 .
\]
Equivalently, every node paired to \(X_i\) leaves a nontrivial infinitesimal
signature in the response of the node-\(i\) intervention channel, and
non-paired leave no such signature.

\end{enumerate}
\end{assumption}

\begin{remark}[No hidden common bath is structural]
\label{rem:no-common-bath-structural}
The decomposition in Definition~\ref{def:hcm-detailed} contains only local bath
couplings of the form
\[
H_i^{\mathrm{bath}}(x_i,e_i,\lambda_i).
\]
Hence no bath variable \(E_i\) directly couples to several observed nodes.
This is not an additional assumption but part of the HCM architecture. Causally,
it plays the role of excluding latent common causes that enter through a shared
environmental degree of freedom.
\end{remark}

\begin{remark}[Independent mechanisms are encoded by the decomposition]
\label{rem:independent-mechanisms-structural}
The node-wise Hamiltonian decomposition
\[
H
=
\sum_i H_{E_i}
+
\sum_i
\left(
H_i^{\mathrm{self}}
+
H_i^{\mathrm{mech}}
+
H_i^{\mathrm{bath}}
\right)
\]
is the HCM analogue of an independent-mechanisms or causal Markov
factorization. The mechanism of node \(i\) consists of its intrinsic term, its
coupling to observed nodes, and its coupling to its local bath. Intervening on
\(\lambda_i\) modifies this node-\(i\) mechanism, but does not modify the
symplectic form, the Hamilton equations, or the mechanisms assigned to other
nodes.
\end{remark}

\begin{remark}[Faithfulness is not implied by the decomposition]
\label{rem:faithfulness-not-structural}
Mechanism dependence is defined by whether \(X_j\) appears in
\(H_i^{\mathrm{mech}}\). Interventional faithfulness additionally requires that
the available intervention coordinate \(\lambda_i\) probes the part of the
node-\(i\) mechanism in which that variable appears. Without this condition, a
mechanism dependence could be present in the Hamiltonian but invisible to work-rate or
local-response tests because the chosen intervention parametrization does not
affect the relevant coupling. 
\end{remark}

\begin{remark}
    The decomposition of $H$ is part of the causal model specification. As in SCMs, the causal graph is not inferred from an observational object alone; it is encoded by the modular decomposition of the mechanisms and by the associated intervention coordinates.
\end{remark}

\paragraph{Causal role of the construction.}
The HCM decomposition is the causal content of the model. The local bath
structure plays the role of node-wise exogenous noise; the absence of common
bath terms excludes hidden common environmental causes; and the node-wise
Hamiltonian decomposition plays the role of a causal mechanism factorization.
The distinction from a classical SCM is that mechanisms are not assignment
functions. They are Hamiltonian energy terms which, together with the fixed
symplectic law, generate trajectories. Interventions therefore modify
Hamiltonian mechanisms, but not the background equations of motion.

\paragraph{Physical role of the assumptions.}
The standing hypotheses are the physical regularity conditions needed to turn
the Hamiltonian construction into well-defined thermodynamics. Well-posedness
ensures that every admissible intervention induces a path law. Reservoir
initialization and local detailed balance identify bath energy changes with
entropy exchange. Microscopic reversibility supplies the backward experiment
against which the forward path law is compared. Together, these are the
standard ingredients behind entropy production as a path-likelihood ratio and
the fluctuation-theorem formulation of the second law
\citep{Jarzynski1997NEQ,Crooks1999FT,Seifert2012Review}.

\begin{remark}[Canonical coordinates]
\label{rem:darboux}
By Darboux's theorem, every point of a \(2m\)-dimensional symplectic manifold
has local coordinates
\[
(q_1,\ldots,q_m,p_1,\ldots,p_m)
\]
such that
\[
\omega=\sum_{k=1}^m dq_k\wedge dp_k.
\]
In these coordinates the Hamiltonian vector-field equation
\(\omega(Y_{H_\lambda},\cdot)=dH_\lambda\) becomes the familiar Hamilton
equations
\[
\dot q_k=\frac{\partial H_\lambda}{\partial p_k},
\qquad
\dot p_k=-\frac{\partial H_\lambda}{\partial q_k}.
\]
Thus the abstract symplectic formulation used in Definition~\ref{def:hcm-detailed}
reduces locally to standard Hamiltonian mechanics.
\end{remark}

\begin{remark}[Reservoir stationarity and system--bath correlations]
\label{rem:reservoir-stationarity}
The reservoir condition should be understood at the thermodynamic scale at
which the bath acts as an equilibrium noise source. It does not require the
joint law of \((X_i(t),E_i(t))\) to factorize for \(t>0\). In fact,
system--bath coupling generally creates correlations along trajectories, and
these correlations are precisely what mediate friction, noise, and heat
exchange. What is assumed is that the bath is large enough, or weakly enough
coupled, that its marginal thermodynamic state remains effectively fixed by
\(\beta_i\) over the time horizon of interest.
\end{remark}

\begin{remark}[Adaptive interventions and nonanticipativity]
The measurability condition on \(\lambda_t\) is the continuous-time analogue of
requiring interventions to depend only on information available before they are
applied. It allows both open-loop protocols and feedback interventions, such as
controllers whose action depends on the observed history \(X_{[0,t]}\). The
additional smoothness assumptions imposed in later results are theorem-specific:
they are needed only when differentiating the protocol, defining an
instantaneous work rate, or constructing infinitesimal ramp interventions. Thus
adaptive interventions belong to the HCM language, while particular
thermodynamic identities may be stated for the smooth subclass where the
calculus is well-defined.
\end{remark}

\section{Proofs}
\label{app:proofs}

\subsection{Proof of Proposition \ref{prop:work-rate-parents}}

\begin{assumption}[Local work-rate faithfulness]
\label{ass:local-work-rate-faithfulness}
Fix a node \(i\), a probing time \(t^*\), and write
\[
\lambda_*:=\lambda(t^*),
\qquad
u_i:=\dot\lambda_i(t^*)\in T_{\lambda_{*,i}}\mathcal L_i .
\]
We say that the \(i\)-probe is locally work-rate faithful at
\((t^*,\lambda_*)\) if, for every \(j\neq i\),
\[
j \in \mathcal{D}_i
\quad\Longleftrightarrow\quad
\left\langle
\partial_{x_j}\partial_{\lambda_i}H(\cdot,\lambda_*),
u_i
\right\rangle
\not\equiv0 .
\]
Here \(\not\equiv0\) means nonzero as a function of the phase-space variables.
When \(\mathcal L_i\) is one-dimensional and \(u_i\neq0\), this reduces to
\[
j \in \mathcal{D}_i
\quad\Longleftrightarrow\quad
\partial_{x_j}\partial_{\lambda_i}H(\cdot,\lambda_*)\not\equiv0 .
\]
\end{assumption}

\begin{proof}
At the probing time \(t^*\), only the \(i\)-th intervention coordinate has
nonzero velocity. Therefore
\[
\dot W(t^*,z;\lambda)
=
\left\langle
\partial_{\lambda_i}H(z,\lambda_*),
\dot\lambda_i(t^*)
\right\rangle .
\]
Consequently, for \(j\neq i\),
\[
\left[\nabla_x\dot W(t^*,z;\lambda)\right]_j
=
\left\langle
\partial_{x_j}\partial_{\lambda_i}H(z,\lambda_*),
\dot\lambda_i(t^*)
\right\rangle .
\]
Thus there exists \(z^*\) such that
\[
\left[\nabla_x\dot W(t^*,z^*;\lambda)\right]_j\neq0
\]
if and only if
\[
\left\langle
\partial_{x_j}\partial_{\lambda_i}H(\cdot,\lambda_*),
\dot\lambda_i(t^*)
\right\rangle
\not\equiv0 .
\]
By local work-rate faithfulness,
Assumption~\ref{ass:local-work-rate-faithfulness}, this is equivalent to
\(j\in\mathcal D_i\)).
\end{proof}

\subsection{Proof of Proposition \ref{thm:proj-ep-witness}}

\begin{proof}
Assume, for contradiction, that \(X_j\) has no causal effect on \(X_i\) under the
two policies. Then
\[
\mathbb P_i^{\lambda^{(1)}}
=
\mathbb P_i^{\lambda^{(2)}}.
\]
Since \(\mathcal R_i\) is fixed, this implies
\[
(\mathcal R_i)_\#\mathbb P_i^{\lambda^{(1)}}
=
(\mathcal R_i)_\#\mathbb P_i^{\lambda^{(2)}}.
\]
Therefore
\[
D_{\mathrm{KL}}\!\left(
\mathbb P_i^{\lambda^{(1)}}
\,\middle\|\,
(\mathcal R_i)_\#\mathbb P_i^{\lambda^{(1)}}
\right)
=
D_{\mathrm{KL}}\!\left(
\mathbb P_i^{\lambda^{(2)}}
\,\middle\|\,
(\mathcal R_i)_\#\mathbb P_i^{\lambda^{(2)}}
\right),
\]
that is,
\[
\mathbb E_{\mathbb P_i^{\lambda^{(1)}}}
\left[\Sigma_i^{\mathrm{proj},\lambda^{(1)}}\right]
=
\mathbb E_{\mathbb P_i^{\lambda^{(2)}}}
\left[\Sigma_i^{\mathrm{proj},\lambda^{(2)}}\right].
\]
This contradicts the assumed strict inequality. Hence
\[
\mathbb P_i^{\lambda^{(1)}}
\neq
\mathbb P_i^{\lambda^{(2)}},
\]
so \(X_j\) has a causal effect on \(X_i\).
\end{proof}

\subsection{Proof of Theorem~\ref{thm:discovery-main}}
\label{app:discovery-proof}

We state the regularity and faithfulness assumptions used by the local
discovery result. The first two assumptions are analytic regularity
conditions ensuring that the channel entropy-production rate is differentiable
at the intervention point. The last two assumptions are the HCM faithfulness
and non-degeneracy requirements specialized to the Langevin setting.

\begin{assumption}[Regularity and faithfulness for local discovery]
\label{ass:discovery-compact}
Fix \(t^*\in(0,T)\), and write
\[
p_*:=p(\cdot,t^*),
\qquad
\lambda_*:=\lambda(t^*).
\]
We assume the following conditions hold.

\begin{enumerate}[label=\textnormal{(D\arabic*)},leftmargin=1.7em,itemsep=0.25em]

\item \textbf{Regular overdamped Langevin dynamics.}
The observed process is
\[
dX_t
=
-\nabla_x U(X_t,\lambda(t))\,dt
+
\sqrt{2D(\lambda(t))}\,dW_t,
\qquad
D(\lambda)=\operatorname{diag}(D_1(\lambda_1),\ldots,D_n(\lambda_n)),
\]
with \(U\in C^3(\mathbb R^n\times\Lambda)\). Each \(D_i\in C^2(\Lambda_i)\)
is bounded above and away from zero on the protocol range.

\item \textbf{Density and current regularity.}
The law of \(X_{t^*}\) has a strictly positive density \(p_*\). For the
right-local continuations used below, the densities \(p^v(\cdot,s)\) solve the
Fokker--Planck equation classically near \(s=0^+\), with probability current
\[
J_i[p,\lambda](x)
=
-\partial_{x_i}U(x,\lambda)p(x)
-
D_i(\lambda_i)\partial_{x_i}p(x).
\]
The channel entropy-production functional
\[
F_i[p,\lambda]
:=
\int_{\mathbb R^n}
\frac{J_i[p,\lambda](x)^2}
{D_i(\lambda_i)p(x)}
\,dx
\]
is finite and differentiable near \((p_*,\lambda_*)\), and differentiation
under the integral sign is valid for frozen-law derivatives in \(\lambda\).

\item \textbf{Right-local probing intervention.}
For each probe node \(j\), the admissible policy class contains two
right-local continuations initialized at \((p_*,\lambda_*)\):
\[
\lambda^0(s)=\lambda_*,
\qquad
\lambda^v(s)=\lambda_*+vse_j+o(s)
\quad\text{as }s\downarrow0,
\]
with \(p^0(\cdot,0)=p^v(\cdot,0)=p_*\). The two continuations are compared as
ensemble laws; no pathwise coupling is assumed.

\item \textbf{Local form of HCM faithfulness.}
The mechanistic dependence induced by the HCM decomposition is locally faithful to
the Langevin potential at \(\lambda_*\): for every \(i\neq j\),
\[
i\in \mathcal D_j
\quad\Longleftrightarrow\quad
\partial_{x_i}\partial_{\lambda_j}U(\cdot,\lambda_*)\not\equiv0
\quad\text{in }L^2(dx).
\]
Equivalently, non-paired have no mixed structural response to the
\(j\)-intervention channel, while every true parent has a nontrivial mixed
response.

\item \textbf{Nonequilibrium non-degeneracy.}
For every \(i\in \mathcal D_j \) with \(i\neq j\),
\[
J_i[p_*,\lambda_*]\in L^2(dx),
\qquad
\partial_{x_i}\partial_{\lambda_j}U(\cdot,\lambda_*)\in L^2(dx),
\]
and
\[
\left\langle
J_i[p_*,\lambda_*],
\partial_{x_i}\partial_{\lambda_j}U(\cdot,\lambda_*)
\right\rangle_{L^2(dx)}
\neq 0 .
\]

\end{enumerate}
\end{assumption}

\begin{proof}
Throughout the proof, all spatial integrals are over \(\mathbb R^n\). We use
the local entropy-production rate identity from
Lemma~\ref{lem:local-ep-current}, namely
\[
\dot\Sigma_i^{\mathrm{loc}}[p,\lambda]
=
F_i[p,\lambda]
:=
\int
\frac{J_i[p,\lambda](x)^2}
{D_i(\lambda_i)p(x)}
\,dx,
\]
where
\[
J_i[p,\lambda](x)
=
-\partial_{x_i}U(x,\lambda)p(x)
-
D_i(\lambda_i)\partial_{x_i}p(x).
\]

Recall that the local response is the right-local contrast
\[
R_{i\mid j}(t^*)
:=
\lim_{v\to0}\lim_{s\downarrow0}
\frac{
F_i[p^v(\cdot,s),\lambda^v(s)]
-
F_i[p^0(\cdot,s),\lambda^0(s)]
}{vs},
\]
whenever the iterated limit exists. Fix \(i\neq j\). The baseline and probe
continuations start from the same ensemble state and protocol value:
\[
p^0(\cdot,0)=p^v(\cdot,0)=p_*,
\qquad
\lambda^0(0)=\lambda^v(0)=\lambda_*.
\]
Moreover,
\[
\lambda^0(s)=\lambda_*,
\qquad
\lambda^v(s)=\lambda_*+vse_j+o(s).
\]

At \(s=0\), the Fokker--Planck velocity of the density depends on the protocol
level \(\lambda_*\), but not on the outgoing protocol velocity. Hence
\[
\left.\partial_s p^v(\cdot,s)\right|_{s=0^+}
=
\left.\partial_s p^0(\cdot,s)\right|_{s=0^+}.
\]
By differentiability of \(F_i\), the density-variation terms therefore cancel
in the first-order contrast defining \(R_{i\mid j}(t^*)\). The only surviving
term is the frozen-law derivative with respect to \(\lambda_j\):
\[
R_{i\mid j}(t^*)
=
\partial_{\lambda_j}^{\mathrm{fr}}F_i[p_*,\lambda_*],
\]
where \(\partial_{\lambda_j}^{\mathrm{fr}}\) denotes differentiation in
\(\lambda_j\) while keeping \(p\) fixed.

It remains to compute this derivative. Since \(i\neq j\), the diagonal
diffusion coefficient \(D_i(\lambda_i)\) is independent of \(\lambda_j\). Thus
\[
\partial_{\lambda_j}^{\mathrm{fr}}F_i[p,\lambda]
=
\int
\frac{2J_i[p,\lambda](x)}
{D_i(\lambda_i)p(x)}
\,
\partial_{\lambda_j}^{\mathrm{fr}}J_i[p,\lambda](x)
\,dx .
\]
Holding \(p\) fixed,
\[
\partial_{\lambda_j}^{\mathrm{fr}}J_i[p,\lambda](x)
=
-\partial_{x_i}\partial_{\lambda_j}U(x,\lambda)\,p(x),
\]
and therefore
\[
\partial_{\lambda_j}^{\mathrm{fr}}F_i[p,\lambda]
=
-\frac{2}{D_i(\lambda_i)}
\int
J_i[p,\lambda](x)
\partial_{x_i}\partial_{\lambda_j}U(x,\lambda)
\,dx .
\]
Evaluating at \((p_*,\lambda_*)\) gives
\[
R_{i\mid j}(t^*)
=
-\frac{2}{D_i(\lambda_{*,i})}
\int
J_i[p_*,\lambda_*](x)
\partial_{x_i}\partial_{\lambda_j}U(x,\lambda_*)
\,dx .
\]

If \(i\notin \mathcal D_j\), local HCM faithfulness gives
\[
\partial_{x_i}\partial_{\lambda_j}U(\cdot,\lambda_*)\equiv0,
\]
and hence \(R_{i\mid j}(t^*)=0\). Conversely, if
\(i\in \mathcal D_j\), then local HCM faithfulness gives
\[
\partial_{x_i}\partial_{\lambda_j}U(\cdot,\lambda_*)\not\equiv0,
\]
and the nonequilibrium non-degeneracy assumption makes the displayed inner
product nonzero. Since \(D_i(\lambda_{*,i})>0\), this implies
\(R_{i\mid j}(t^*)\neq0\). Therefore
\[
i\in \mathcal D_j
\quad\Longleftrightarrow\quad
R_{i\mid j}(t^*)\neq0 .
\]
\end{proof}

\begin{remark}[Genericity of response non-degeneracy]
For fixed nonzero
\[
g_{ij}:=\partial_{x_i}\partial_{\lambda_j}U(\cdot,\lambda_*),
\]
the set
\[
\left\{
J_i\in L^2(dx):
\langle J_i,g_{ij}\rangle_{L^2(dx)}=0
\right\}
\]
is a closed codimension-one hyperplane in \(L^2(dx)\). Thus
Assumption~\ref{ass:discovery-compact}(D5) is a generic nonequilibrium
condition: it fails only when the current is orthogonal to the structural mixed
derivative, for instance when \(J_i(\cdot,t^*)\equiv0\).
\end{remark}

\subsection{Proof of Theorem~\ref{thm:realization-informal}}
\label{app:realization-proof}

The realization result used in Theorem~\ref{thm:realization-informal} is a
standard consequence of the oscillator-bath derivation of generalized Langevin
equations and of the Markovian limits. We recall the general construction
only to make explicit that it is compatible with the HCM factorization. For an example of proof along with a contextualization of the terminology used here, we refer the reader to \citep{Sekimoto2010}.

Consider the Hamiltonian
\[
H^{N,m}(z,\lambda)
=
\sum_i \frac{p_i^2}{2m}
+
U(x,\lambda)
+
\sum_i
\left[
H_{E_i}^N(e_i)
+
H_{X_iE_i}^N(x_i,e_i,\lambda_i)
\right],
\]
where each $E_i$ is an independent oscillator bath initialized at a Gibbs law.
For Ohmic spectral densities, the marginal dynamics of $x$ converges, in the
Markovian bath limit and the overdamped limit, to the controlled Langevin
diffusion
\[
dX_t
=
-\nabla_x U(X_t,\lambda_t)\,dt
+
\sigma(\lambda_t)\,dW_t,
\]
after choosing unit mobility and bath temperatures satisfying
$D_i(\lambda_i)=\sigma_i^2(\lambda_i)/2$.
This is the classical Caldeira-Leggett construction of Langevin
dynamics from Hamiltonian system--bath dynamics.

\begin{proof}[Proof sketch of Theorem~\ref{thm:realization-informal}]
The construction above satisfies the HCM axioms: the initial bath law
factorizes across nodes, each bath couples only to its corresponding observed
coordinate, and the intervention coordinate $\lambda_i$ enters only the
$i$-th local Hamiltonian terms. Standard oscillator-bath theory gives a
generalized Langevin equation with memory kernel determined by the bath
spectral density. In the Ohmic limit the memory kernel converges to a delta
kernel and the random bath force converges to white noise satisfying the
fluctuation--dissipation relation. The subsequent limit \citep{CaldeiraLeggett1983QBM} yields
the overdamped Itô diffusion in Theorem~\ref{thm:realization-informal}.
\end{proof}

\subsection{Entropy Production in Langevin Dynamics}
\label{app:otsubo-channel-ep}

The identity below is standard in stochastic thermodynamics; see, e.g.,
\citet{Seifert2012Review,HorowitzEsposito2014InfoFlow,Otsubo_2020}.
We include the short proof only for completeness and to make explicit the
quantity used in Theorem~\ref{thm:discovery-main} and in the experiments.

\begin{lemma}[Channel-local entropy-production rate]
\label{lem:local-ep-current}
Consider the controlled overdamped Langevin diffusion
\[
dX_t
=
b(X_t,\lambda_t)\,dt
+
\sqrt{2D(\lambda_t)}\,dW_t,
\qquad
b_i(x,\lambda)=-\partial_{x_i}U(x,\lambda),
\]
with diagonal diffusion
\[
D(\lambda)=\operatorname{diag}(D_1(\lambda_1),\ldots,D_n(\lambda_n)).
\]
Assume the joint law of \(X_t\) has a smooth strictly positive density
\(p^\lambda(x,t)\) with sufficient decay at infinity. Let
\[
J_i^\lambda(x,t)
:=
b_i(x,\lambda_t)p^\lambda(x,t)
-
D_i(\lambda_i(t))\,\partial_{x_i}p^\lambda(x,t)
\]
be the \(i\)-th probability current. Then the channel-local entropy-production
rate of node \(i\) is
\[
\dot\Sigma_i^{\mathrm{loc}}[p^\lambda(\cdot,t),\lambda_t]
=
\int_{\mathbb R^n}
\frac{
\bigl(J_i^\lambda(x,t)\bigr)^2
}{
D_i(\lambda_i(t))\,p^\lambda(x,t)
}
\,dx .
\]
\end{lemma}

\begin{proof}
We suppress the superscript \(\lambda\) and the time argument. The Fokker--Planck
equation is
\[
\partial_t p(x,t)
=
-\sum_{k=1}^n \partial_{x_k}J_k(x,t).
\]
The contribution of the \(i\)-th current to the joint system-entropy rate is
\[
\dot s_{\mathrm{sys}}^{(i)}(t)
=
-\int_{\mathbb R^n}
J_i(x,t)\,\partial_{x_i}\log p(x,t)\,dx .
\]
With the sign convention that \(Q_i\) denotes heat transferred from bath
\(E_i\) into the observed system, local detailed balance gives the corresponding
medium-entropy contribution in the form
\[
-\beta_i\dot Q_i(t)
=
\int_{\mathbb R^n}
\frac{b_i(x,\lambda_t)}{D_i(\lambda_i(t))}
J_i(x,t)\,dx .
\]
This is the standard overdamped thermodynamic-force expression; equivalently,
one may view \(D_i\) as the mobility-temperature product in units where the
mobility has been absorbed into the drift.
By Definition~\ref{def:local-ep}, the information-flow correction cancels the
difference between the marginal entropy rate and the \(i\)-th contribution to
the joint system-entropy rate, hence
\[
\dot\Sigma_i^{\mathrm{loc}}
=
\dot s_{\mathrm{sys}}^{(i)}
-
\beta_i\dot Q_i .
\]
Substituting the two displayed expressions gives
\[
\dot\Sigma_i^{\mathrm{loc}}
=
\int_{\mathbb R^n}
J_i(x,t)
\left[
\frac{b_i(x,\lambda_t)}{D_i(\lambda_i(t))}
-
\partial_{x_i}\log p(x,t)
\right]dx .
\]
Since
\[
\frac{b_i}{D_i}-\partial_{x_i}\log p
=
\frac{b_i p-D_i\partial_{x_i}p}{D_i p}
=
\frac{J_i}{D_i p},
\]
we obtain
\[
\dot\Sigma_i^{\mathrm{loc}}
=
\int_{\mathbb R^n}
\frac{J_i(x,t)^2}{D_i(\lambda_i(t))p(x,t)}
\,dx ,
\]
as claimed.
\end{proof}

\begin{remark}[Projected entropy-production estimation]
\label{rem:projected-ep-estimation}
For the projected entropy-production experiments, we feed only the partial
trajectory \(x_i(\cdot)\) to the variational entropy-production estimator of
\citet{Otsubo_2020,Otsubo_2022}. This changes the target of estimation relative
to the full-state channel-local quantity in
Lemma~\ref{lem:local-ep-current}: the estimator can only use currents and test
functions measurable with respect to the observed marginal trajectory of
\(X_i\). Consequently, it estimates the time-reversal asymmetry visible after
projecting the dynamics onto \(X_i\), or equivalently the projection of the
thermodynamic force onto the restricted function class available from
\(x_i(\cdot)\). This is precisely the object needed for the path-level causal
effect witness in Proposition~\ref{thm:proj-ep-witness}, since that proposition
is formulated in terms of the marginal path law
\(\mathbb P_i^\lambda\). The statistical and thermodynamic consequences of
estimating entropy production from partial observations, including the resulting
coarse-grained or projected interpretation, are discussed in the original
estimation work of \citet{Otsubo_2020,Otsubo_2022}.
\end{remark}

\subsection{Separation of Measures}

\begin{proposition}[Causal effect and separating path-law discrepancies]
\label{prop:effect-iff-discrepancy}
Fix \(i\neq j\), a baseline policy \(\lambda_{-j}\), and a discrepancy
\[
\mathfrak D:
\mathcal P\!\left(C([0,T],\Gamma_{X_i})\right)
\times
\mathcal P\!\left(C([0,T],\Gamma_{X_i})\right)
\to \mathbb R .
\]
Assume that \(\mathfrak D\) separates probability measures, in the sense that
\[
\mathfrak D(\mu,\nu)=0
\quad\Longleftrightarrow\quad
\mu=\nu .
\]
Then \(X_j\) has a causal effect on \(X_i\) over \([0,T]\) in the sense of
Definition~\ref{def:hcm-causal-effect} if and only if there exist two
admissible policies
\[
\lambda^{(1)}=(\lambda_j^{(1)},\lambda_{-j}),
\qquad
\lambda^{(2)}=(\lambda_j^{(2)},\lambda_{-j}),
\]
such that
\[
\mathrm{ATE}_{0:T}^{(i)}
\bigl(
\mathfrak D;
\lambda^{(1)},\lambda^{(2)}
\bigr)
\neq 0 .
\]
\end{proposition}

\begin{proof}
By Definition~\ref{def:hcm-causal-effect}, \(X_j\) has a causal effect on
\(X_i\) over \([0,T]\) if and only if there exist two admissible policies
\[
\lambda^{(1)}=(\lambda_j^{(1)},\lambda_{-j}),
\qquad
\lambda^{(2)}=(\lambda_j^{(2)},\lambda_{-j}),
\]
such that the corresponding marginal path laws of \(X_i\) are different:
\[
\mathbb P_i^{\lambda^{(1)}}
\neq
\mathbb P_i^{\lambda^{(2)}} .
\]
Since \(\mathfrak D\) separates measures, this is equivalent to
\[
\mathfrak D
\left(
\mathbb P_i^{\lambda^{(1)}},
\mathbb P_i^{\lambda^{(2)}}
\right)
\neq 0 .
\]
By Definition~\ref{def:hcm-ate}, the left-hand side is exactly
\[
\mathrm{ATE}_{0:T}^{(i)}
\bigl(
\mathfrak D;
\lambda^{(1)},\lambda^{(2)}
\bigr).
\]
Hence \(X_j\) has a causal effect on \(X_i\) if and only if the generalized
trajectory ATE is nonzero for some pair of admissible policies.
\end{proof}

\section{Experimental Details}
\label{app:experiments}

\subsection{Comparison of Entropy Production and other ATEs}

We test the central claim that the entropy production difference:

$$ATE(\Sigma_i^\mathrm{proj}):=\mathbb E_{\mathbb P_i^{\lambda^{(1)}}}\left[\Sigma_i^{\mathrm{proj}, \lambda^{(1)}}\right] - \mathbb E_{\mathbb P_i^{\lambda^{(2)}}}\left[\Sigma_i^{\mathrm{proj}, \lambda^{(2)}}\right]$$

is a causal estimand that detects effects which are \emph{invisible} to the
two standard observable-level estimands: the endpoint average treatment
effect $ATE = \mathbb{E}[Y(T)\mid \mathrm{do}(1)] - \mathbb{E}[Y(T)\mid \mathrm{do}(2)]$ and the
cumulative ATE $cATE = \int_0^T \mathbb{E}[Y(t)\mid \mathrm{do}(1)] - \mathbb{E}[Y(t) \mid  \mathrm{do}(2)]\,dt$.
The two systems below are \emph{designed} so that one of these standard
estimands is provably zero under the null (same-endpoint and
time-translation symmetry, respectively), while the underlying dynamics
are demonstrably non-equivalent in a global sense.


\subsubsection{Experiment A: Same-Endpoint Protocols on a 10-Node DAG}
\label{sec:exp_a_updated}

\paragraph{Data-generating process}
Ten coupled overdamped Langevin oscillators on the linear chain
$0\!\to\!1\!\to\!\cdots\!\to\!9$ 

\[
  dq_i = \Bigl[-k\bigl(q_i - a_i(t)\bigr) + \kappa\!\!\sum_{j\in\mathrm{adj}(i)}\!\!(q_j - q_i)\Bigr]dt
         + \sigma\,dW_i,
\]

with stiffness $k=5$,
nearest-neighbour coupling $\kappa=1.5$, and noise amplitude $\sigma=1$.
The relaxation time of the driven node is
$\tau_0 = 1/(k+\kappa) \approx 0.15$, much smaller than the protocol horizon
$T=10$, so all nodes are in the quasi-static regime with respect to the
control.
Initial conditions are drawn $q(0)\!\sim\!\mathcal{N}(\mathbf{0},\sigma_q^2 I)$
with $\sigma_q=0.3$. We integrate $N=1000$ trajectories with
Euler--Maruyama at $\Delta t=0.01$ ($n_{\rm steps}=1000$).

\paragraph{Two same-endpoint protocols.}
Both policies act on node $0$ only, with $a_0(T)=a_{\rm end}=1$:
\[
  \lambda_{1}\ \text{(ramp)}\!:\ a_0(t)=t/T,\qquad
  \lambda_{2}\ \text{(excursion)}\!:\ a_0(t)=
   \begin{cases} 4\cdot 2t/T & t\le T/2,\\
                 4 + (1{-}4)\cdot \tfrac{2(t-T/2)}{T} & t>T/2.\end{cases}
\]
Because $T\!\gg\!\tau_0$ and $a_0(T)$ matches across protocols, the
endpoint distribution $\mu_T$ is identical in the limit $T/\tau_0\!\to\!\infty$
for every node, hence the endpoint ATE on the observable $Y(t)=q_0(t)$
vanishes by construction.

\paragraph{Estimation of Entropy Production}
We use NEEP\ \citep{Kim_2020, Otsubo_2020}.
Architecture: $L=3$ fully-connected hidden layers of width $H=64$;
optimiser Adam, learning rate $10^{-3}$, $n_{\rm step}=1000$ gradient steps;
trained on the single-coordinate marginal $\{q_0(t)\}$ (the directly
forced node).
We run $K=3$ independent seeds and report mean $\pm$ standard deviation.

\subsubsection{Experiment B: Protocol-Order Effect on a 3-Node Hidden-Mediator System}
\label{sec:exp_b_updated}

\paragraph{Data-generating process}
Three overdamped Langevin nodes $q=[q_1,q_2,s]$ with a hidden nonlinear
mediator $s$, defined by:

\begin{align*}
  dq_1 &= \bigl[-k(q_1 - \lambda_a(t)) - \gamma_c\,s\,q_2 - \delta_c\,s\bigr]\,dt + \sigma\,dW_1, \\
  dq_2 &= \bigl[-k(q_2 - \lambda_b(t)) - \gamma_c\,s\,q_1\bigr]\,dt + \sigma\,dW_2, \\
  ds   &= \bigl[-k_s\,s - \gamma_c\,q_1 q_2 - \delta_c\,q_1\bigr]\,dt + \sigma_s\,dW_s.
\end{align*}

Parameters: $k=k_s=4$ (so $\tau_q=\tau_s=1/4=0.25$), symmetric coupling
$\gamma_c=0.05$ (small, so non-cancelling nonlinear terms are negligible),
asymmetric coupling $\delta_c=1.0$ (the primary driver of entropy production asymmetry,
which scales as $\delta_c^2$), $\sigma=\sigma_s=1$.
Initial conditions $q_1(0),q_2(0),s(0)\sim\mathcal{N}(0,0.3^2)$.
We integrate $N=2000$ trajectories at $\Delta t=0.01$, $T=8$
($n_{\rm steps}=800$).

\paragraph{Two protocol orderings.}
Two raised-cosine pulses of amplitude $A=3$, each of duration $T/2$,
applied in opposite order:
\[
  {AB}\!:\ \lambda_a\text{ on }[0,T/2],\ \lambda_b\text{ on }[T/2,T];\qquad
  {BA}\!:\ \lambda_b\text{ on }[0,T/2],\ \lambda_a\text{ on }[T/2,T].
\]
The total control effort and the marginal time-budget per channel are
identical; only the order differs.
For the observable $Y=q_1+q_2+s$ a time-translation argument gives
$cATE(Y)=\int_0^T \mathbb{E}[Y_t\mid {AB}]-\mathbb{E}[Y_t\mid {BA}]\,dt = 0$
up to $O(\gamma_c)$ corrections.

\paragraph{Estimator}
NEEP\ as in Experiment A, but trained on the full 3-D state
$q=(q_1,q_2,s)$ to capture the mediator-driven asymmetry.
Architecture and optimiser as before; $K=5$ independent seeds.

\subsubsection{Hyperparameters and Reproducibility}

We provide detailed hyperparameters for reproducibility in the attached code.

\paragraph{Compute}
Each experiment fits in under $10$ minutes wall-clock on a single GPU
(mostly in NEEP training).

\paragraph{Limitations and scope}
 Both designs are deliberately small physical systems chosen to make
the analytic ATE-cancellation argument transparent; we leverage the
quasi-static regime ($T\!\gg\!\tau$) in Exp.~A and the time-translation
symmetry of the integrated observable in Exp.~B.
The point of these experiments is \emph{not} scale, but to demonstrate that
entropy production resolves causal asymmetries that any observable-level estimand built
on $Y$ will miss \emph{by construction}.
This can be taken as a lesson in selecting the correct quantities to estimate in temporal notions of causality: while marginal and cumulative notions of the ATE might be correct for certain applications, they are not fully descriptive of path-level constructions, where instead entropy production can help.

\subsection{Mechanism Identification}

We evaluate our method for mechanism identification on \emph{random cyclic} interaction graphs
under three families of $\lambda$-gated coupling potentials of
increasing nonlinearity: a \textbf{bilinear} reference model, a
saturating \textbf{tanh-sigmoid} model, and a quartic-confined
\textbf{cubic} model.  All three families share the same two-intervention response protocol, the same
neural-network EP estimator, and the same fixed thresholding rule
(\S\ref{sec:cdep_metrics}).

\subsubsection{Data-Generating Processes}
\label{sec:dgp}

\paragraph{Random graphs (cyclic)}
For each system family we draw $G=10$ independent Erd\H{o}s--R\'enyi
\emph{directed} graphs on $n=15$ nodes with edge probability
$p_{\rm edge}=0.10$, self-loops removed.  In the realized draws, each graph has between 20 and
28 directed edges and typically contains at least one directed cycle (>90\%).
For each graph we pick a target node $j$ with $|\mathcal{D}_j|\!\ge\!2$ if
possible (else $\ge\!1$) by uniform sampling among candidate nodes
with the required set size, using a deterministic seed schedule
(base seed $20240501$, increment $7919$ per graph) so that the same
underlying graph topology and target $j$ are used across the three
system families. All three families share the structure
\begin{equation}
  \label{eq:sde_scaffold}
  dX_i \;=\; -\nabla_x U\,dt
              \;+\; \sigma_0\,dW_i,
  \qquad i=1,\dots,n,
\end{equation}
with $\sigma_0=0.5$ and self-stiffness coefficients
$a_i\sim\mathrm{Unif}(a_{\rm lo},a_{\rm hi})$ .
The drift $-\nabla_x U$ comes from a $\lambda$-forced potential
\begin{equation}
  \label{eq:lambda_potential}
  U(\mathbf{x},\lambda)
  \;=\; \tfrac12\sum_i a_i x_i^2
        \;+\; U_{\rm self}(\mathbf{x})
        \;+\; \sum_j \lambda_j\!\left[\,b_j x_j
              + \!\sum_{i\in \mathcal{D}_j} c_{ij}\,\phi(x_i,x_j)\right],
\end{equation}
with own-coupling $b_j\sim\mathrm{Unif}(0.5,1.5)\cdot\{\pm1\}$ and edge
weights $c_{ij}=\frac{1}{2}\!\cdot\!\mathrm{Unif}(1.5,3.5)\cdot\{\pm1\}$ on graph
edges (zero off-edges). The three
families differ in the coupling kernel $\phi$ and the
self-confinement $U_{\rm self}$:
\begin{itemize}[leftmargin=*,itemsep=2pt]
  \item \textbf{Bilinear} (\texttt{lambda\_coupled}):
        $U_{\rm self}\!=\!0$, $\phi(x_i,x_j)=x_i x_j$.  Drift remains
        linear in $\mathbf{x}$ at any fixed $\lambda$;
        marginals are Gaussian.
  \item \textbf{Sigmoid} (\texttt{lambda\_sigmoid}, $\beta=2.5$):
        $U_{\rm self}\!=\!0$, $\phi(x_i,x_j)=\tanh(\beta x_i)\,x_j$.
        The input saturates, yielding a strongly nonlinear
        response of the child to large parent excursions.
  \item \textbf{Cubic} (\texttt{lambda\_cubic}, $\gamma=0.5$):
        $U_{\rm self}=\frac{\gamma}{4}\sum_i x_i^4$,
        $\phi(x_i,x_j)=x_i x_j$.  The quartic self-confinement makes
        the marginal stationary distribution distinctly non-Gaussian.
\end{itemize}

\paragraph{Two-intervention response protocol}
We estimate the response matrix $R_{i\mid j}$ via the corrected
two-intervention design as implied by \ref{thm:discovery-main}.  Starting from a
common burn-in snapshot of the null system
($T_{\rm burn}=4.0$, $\Delta t_{\rm burn}=0.02$, $N$ trajectories), we
simulate two ensembles on a short window of length
$s_{\rm window}=0.4$ (with $n_{\rm steps}=40$ Euler-Maruyama
substeps, $\Delta t=0.01$) under
\begin{align}
  \lambda^0(s)  &= \lambda^{\!\star},
    & &\text{(frozen)} \\
  \lambda^v(s)  &= \lambda^{\!\star}
                                 + v\,s\,\mathbf{e}_j,
    & &\text{(probe)}
\end{align}
with response velocity $v=0.5$ and \emph{independent} Brownian noise
streams. Under this single-coordinate intervention
the analytic ground-truth response $R^{\rm true}_{i\mid j}$ of every
non-paired node vanishes \emph{exactly}, giving a clean
binary classification problem.  For each candidate $i\!\neq\!j$ we
fit a single channel-EP estimator on the concatenated frozen+probe
trajectories and read
\begin{equation}
  \hat R_{i\mid j}
  \;=\; \mathrm{slope}_s
        \!\bigl[\hat\sigma^v_i(s) - \hat\sigma^0_i(s)\bigr]\,/\,v,
\end{equation}
where the slope is a least-squares fit through the origin over the
$n_{\rm steps}$ midpoints of the short window.

\subsubsection{NEEP Training}
\label{sec:neep}

The per-coordinate, time-resolved EP rate $\hat\sigma_i(t)$ is
obtained with the non-stationary local NEEP estimator
\citep{Otsubo_2020} implemented as a feed-forward
network with a temporal kernel basis (\texttt{FNNKt}) using the public \texttt{LearnEntropy} library: the input is
the joint state $x\in \mathbb{R}^{n}$, the output is the time-resolved
scalar EP rate of coordinate $i$.  We use the same architecture and
schedule for every node:
\begin{itemize}[leftmargin=*,itemsep=2pt]
  \item Architecture: $L=2$ hidden layers, width $H=64$,
        $n_{\rm output}=20$ temporal kernel centres
        \citep{Kim_2020}.
  \item Optimiser: Adam, learning rate $5\!\times\!10^{-4}$,
        weight decay $10^{-3}$, $2000$ gradient-ascent steps.
  \item Train/validation split: $70/30$ along the trajectory axis;
        best-on-validation checkpoint selected.
\end{itemize}
Each graph took $155$--$294$\,s on a single GPU. The full sweep (3 systems $\times$ 10 graphs)
fits in roughly one GPU-hour.

\subsubsection{Metrics}
\label{sec:cdep_metrics}

For each graph we report:
\begin{itemize}[leftmargin=*,itemsep=2pt]
  \item \textbf{Headline F1, precision, recall.}  A node $i\!\neq\!j$
        is predicted as paired to $j$ iff
        \begin{equation}
          \label{eq:threshold_rule}
          |\hat R_{i\mid j}|
          \;>\;
          \tau\,\max_{i'\neq j}|\hat R_{i'\mid j}|,
          \qquad \tau=0.10.
        \end{equation}
        Precision, recall, and F1 are computed against $\mathcal{D}_j$.
  \item \textbf{Localized Structural Hamming Distance.}  We define
        the localized SHD at the target $j$ as the symmetric
        difference between the predicted set and the true
        dependent set,
        \begin{equation}
          \mathrm{SHD}_j
          \;:=\;
          \bigl|\widehat{\mathcal{D}}_j\,\triangle\ \mathcal{D}_j\bigr|
          \;=\; \mathrm{FP}_j + \mathrm{FN}_j,
        \end{equation}
        scored with the same threshold rule
        \eqref{eq:threshold_rule}.  This is the SHD of the $j$-th
        in-edge column of the adjacency matrix and is the
        finest-grain SHD compatible with our recovery setup.
\end{itemize}
We aggregate across the $G=10$ graphs of each system family by
reporting the mean and a $95\%$ percentile bootstrap interval
($B=20{,}000$ resamples).  The bootstrap samples are over graph
indices, so the intervals reflect graph-to-graph variability at
fixed estimator hyperparameters.

\section{Relationship to Dynamic Structural Causal Models}
\label{app:dscm}

As mathematical objects, dynamic structural causal models (DSCMs)
describe assignments between entire trajectories and can represent
dynamics beyond the gradient-drift Langevin class studied here,
including SDEs with non-gradient drifts
\citep{BoekenMooij2024,boeken2026}.
Such generality requires conditions ensuring that observational and
interventional distributions are well-defined, a non-trivial issue
already for cyclic SCMs without an explicit temporal interpretation
\citep{bongers2021foundations}.
DSCMs do come equipped with interventional semantics, and their
graphical results should therefore not be characterized as merely
observational. Our emphasis is instead on the additional physical
structure and its statistical implications.
The recent work of \citet{boeken2026} provides graphical Markov
properties and do-calculus for systems of causal SDEs, including
cyclic and non-Markovian systems. Its constraint-based discovery
illustrations rely on a conditional-independence oracle, while
statistical aspects of causal effect estimation are explicitly
identified as future work.
HCMs provide a complementary contribution in the Hamiltonian setting:
the connection to thermodynamics supplies physically meaningful,
path-level quantities whose contrasts can witness causal effects
under the conditions stated in this paper, together with
trajectory-based estimation procedures in the Langevin settings
we investigate.
The ease of defining interventional quantities is also a central
advantage of the HCM formulation: interacting, time-dependent
controls enter directly through the Hamiltonian, and the admissible
class of policies specifies the available interventions.
The framework permits cyclic interactions, non-stationary evolution,
and memory in the observed subsystem induced by the baths.
These features are not exclusive to HCMs: DSCMs can also represent
parameter interventions by including the parameters as endogenous
processes \citep{boeken2026}.
HCMs instead make the physical implementation of controls and their
thermodynamic consequences explicit within the model.
Thus, the contribution is additional physical and statistical
structure for the dynamics covered by our assumptions, rather than
greater representational generality than generic DSCMs.

\section{Autonomous System--Demon Coupling}
\label{app:maxwell-demon}

An HCM node comprises an observed subsystem $X_i$, its associated
bath $E_i$, and its control $\lambda_i$. The spaces $\Gamma_X$,
$\Gamma_E$, and $\mathcal L$ have corresponding factorizations,
with the Hamiltonian coupling these components. To illustrate
this structure in the setting of autonomous Maxwell demons
\citep{ItoSagawa2015MaxwellDemon}, consider two nodes, a system $S$
and a demon $D$. At the microscopic level, take
$X_i=(x_i,p_i)\in\mathbb R^2$, with a Caldeira--Leggett bath
$E_i\in\mathbb R^{2N_i}$ prepared at inverse temperature
$\beta_i>0$, as in the realization construction.
Let $\mathcal L_S=\mathcal L_D=\mathbb R$ and choose
\[
H_S^{\mathrm{self}}
=\frac{p_S^2}{2m_S}
+\frac{\kappa_S}{4}(x_S^2-1)^2,
\qquad
H_D^{\mathrm{self}}
=\frac{p_D^2}{2m_D}
+\frac{\kappa_D}{2}x_D^2,
\]
\[
H_S^{\mathrm{mech}}=\lambda_S x_Sx_D,
\qquad
H_D^{\mathrm{mech}}=\lambda_D x_Sx_D,
\]
with $m_i,\kappa_i>0$.
The system has a double-well potential, while the demon's
preferred position depends on the system coordinate and its
state exerts a force back on the system.
Both controls modify the same reciprocal coupling, whose total
strength is $\lambda_S+\lambda_D$; they are not independent
directional gates for measurement and feedback.

Taking $\mathcal A=\{\lambda:\lambda(t)\equiv(1,1)\}$ gives
autonomous dynamics without external modulation.
The mechanism-dependence sets are
$\mathcal D_S=\{D\}$ and $\mathcal D_D=\{S\}$, so the nodes are
mechanically paired. With this singleton admissible class,
however, there are no intervention contrasts and hence no
operational parents under
Definition~\ref{def:hcm-operational-parents}.
Enlarging $\mathcal A$ to include varying control protocols
allows interventional comparisons and externally driven operation.

In the overdamped Langevin limit with unit mobilities and
independent thermal noises, let $U$ denote the total potential,
$p$ the joint density, and
$J_i=-(\partial_{x_i}U)p-\beta_i^{-1}\partial_{x_i}p$.
Under the usual regularity and decay conditions, the local
entropy production of Definition~\ref{def:local-ep} satisfies
\[
\Sigma_i^{\mathrm{loc}}
=
\Delta s_i^{\mathrm{marg}}
-\beta_i\mathbb E[Q_i]
-\int_0^T\dot I_i^{\mathrm{flow}}(t)\,dt
=
\int_0^T\!\int_{\mathbb R^2}
\frac{\beta_i J_i(x,t)^2}{p(x,t)}\,dx\,dt
\geq0.
\]
Consequently, the system's entropy balance before the
information correction obeys
\[
\Delta s_S^{\mathrm{marg}}-\beta_S\mathbb E[Q_S]
\geq
\int_0^T\dot I_S^{\mathrm{flow}}(t)\,dt.
\]
When the system consumes correlations, the right-hand side
can be negative: its entropy balance must therefore be
interpreted together with the information exchanged with
the demon \citep{HorowitzEsposito2014InfoFlow}.
Summing the two local productions recovers the joint second law,
\[
\Sigma_S^{\mathrm{loc}}+\Sigma_D^{\mathrm{loc}}
=
\Delta s_{S,D}
-\beta_S\mathbb E[Q_S]
-\beta_D\mathbb E[Q_D]
\geq0.
\]
This construction illustrates the coupling and information
balance underlying autonomous demon models; demonstrating
a work-extracting operating regime requires additional analysis.

\end{document}